\documentclass[10pt,journal]{IEEEtran}
\IEEEoverridecommandlockouts

\usepackage[dvips]{graphicx}

\usepackage{xspace}
\usepackage{caption}
\usepackage{subcaption}
\usepackage{placeins}
\usepackage{balance}
\usepackage{cite}
\usepackage{bm}
\usepackage{amsmath}
\usepackage{array}
\usepackage{setspace} 
\usepackage[nameinlink]{cleveref}
\usepackage{accents}
\usepackage{mathtools}
\usepackage{algpseudocode}

\usepackage{amsthm} 
\usepackage{multirow}
\usepackage[linesnumbered,ruled,vlined]{algorithm2e}
\usepackage{tabularx}
\SetKwRepeat{Do}{do}{while}
\usepackage{color}
\usepackage{epstopdf}
\usepackage{endnotes}
\usepackage{framed}

\usepackage{cool} 
\usepackage{enumerate} 

\usepackage{pifont}
\usepackage{xspace} 
\Crefname{figure}{Fig.}{Figs.}

\newcommand{\SumNoLim}[2]{\ensuremath{\sum_{#1}^{#2}}}

\usepackage{mathrsfs}
\usepackage[justification=centering]{caption}
\allowdisplaybreaks[4]

\newcolumntype{P}[1]{>{\centering\arraybackslash}p{#1}}
\newcommand{\xmarklight}{\ding{53}}%

\DeclareUnicodeCharacter{2212}{-}

\newcommand{\x}{{x}_{d}}

\newtheorem{theorem}{Theorem}

\usepackage{color}
\def\BibTeX{{\rm B\kern-.05em{\sc i\kern-.025em b}\kern-.08em
    T\kern-.1667em\lower.7ex\hbox{E}\kern-.125emX}}

\begin{document}
\title{Fine-Grained Data Selection for Improved Energy Efficiency of Federated Edge Learning}
\author{Abdullatif Albaseer,~\IEEEmembership{Student~Member,~IEEE,}
        Mohamed Abdallah,~\IEEEmembership{Senior Member,~IEEE,}
     Ala Al-Fuqaha,~\IEEEmembership{Senior Member,~IEEE,}
     Aiman Erbad,~\IEEEmembership{Senior Member,~IEEE}
    \thanks{Abdullatif Albaseer, Mohamed Abdallah Ala Al-Fuqaha and~Aiman Erbad are with Division of Information and Computing Technology, College of science and engineering, Hamad Bin Khlifa University,Doha, Qatar (e-mail:\{amalbaseer, moabdallah, aalfuqaha,AErbad\}@hbku.edu.qa).}
}
\maketitle
\begin{abstract}
In Federated edge learning (FEEL), energy-constrained devices at the network edge consume significant energy when training and uploading their local machine learning models, leading to a decrease in their lifetime.
This work proposes novel solutions for energy-efficient FEEL by jointly considering local training data, available computation, and communications resources, and deadline constraints of FEEL rounds to reduce energy consumption.
This paper considers a system model where the edge server is equipped with multiple antennas employing beamforming techniques to communicate with the local users through orthogonal channels.
Specifically, we consider a problem that aims to find the optimal user's resources, including the fine-grained selection of relevant training samples, bandwidth, transmission power, beamforming weights, and processing speed with the goal of minimizing the total energy consumption given a deadline constraint on the communication rounds of FEEL.  
Then, we devise tractable solutions by first proposing a novel fine-grained training algorithm that excludes less relevant training samples and effectively chooses only the samples that improve the model's performance. 
After that, we derive closed-form solutions, followed by a Golden-Section-based iterative algorithm to find the optimal computation and communication resources that minimize energy consumption.  
Experiments using MNIST and CIFAR-10 datasets demonstrate that our proposed algorithms considerably outperform the state-of-the-art solutions as energy consumption decreases by 79\% for MNIST and 73\% for CIFAR-10 datasets.

\end{abstract}

\begin{IEEEkeywords}
Federated Edge Learning (FEEL), Edge Intelligence, Data selection, Learning Algorithm,  Energy consumption, Convergence rate, Resource allocation.
\end{IEEEkeywords}

\IEEEdisplaynontitleabstractindextext

\IEEEpeerreviewmaketitle
\section{INTRODUCTION}
 The extraordinary improvements in the Internet of Things (IoT), ubiquitous communications, and artificial intelligence (AI) have induced an exponential growth in the size of data generated every day by edge devices (e.g., IoT devices, smartphones, sensors, actuators). According to Cisco, the increase in data generated by people, machines, and things is anticipated to be in the millions of billions of gigabytes,  {nearly $77.5$ exabytes per month by 2022~\cite{forecast2019cisco}. }
 
 Leveraging the proliferation of AI, the generated data, and mobile edge computing (MEC) techniques ~\cite{li2019edge,wang2019edge,zhu2020toward} can bring valuable innovative services to end-users~\cite{lecun2015deep}. 
 Lately, MEC techniques received a lot of attention from practitioners and researchers because of their potential in reducing latency and delivering an elegant quality of experience for end devices \cite{mao2017survey}. 
 It is envisioned that MEC will be an enabling tool for sixth-generation (6G) networks permitting new emerging applications, such as human-centric services, virtual reality, and augmented reality; consequently, realizing the vision of network intelligence\cite{letaief2019roadmap,mo2020energy}. 
 Yet, transferring massive volumes of user data to a central server brings out unwanted communication costs and security risks due to the networks' limitations, scalability issues, inadequate bandwidth, and most importantly, users' privacy. 

Recently, federated edge learning ({FEEL}) has emerged as a potential candidate that utilizes MEC to address these challenges\cite{lim2020federated,guo2020feel}. FEEL can be seen as a cutting-edge collaborative machine learning \textbf{(ML)} technique for future IoT and edge systems\cite{mcmahan2016communication,smith2017federated,saad2019vision,park2019wireless}.
FEEL strives to collaboratively train a shared \textbf{ML} model on client devices while maintaining their privacy since data remains where it is produced, and only resultant model parameters are shared with the server~\cite{lim2020federated}. 
In FEEL rounds, the global model updating task, along with the associated computation and communication phases, expend significant energy even though the edge devices themselves are typically energy-constrained.

Energy constrained edge devices might negatively affect FEEL performance as the battery level may limit the worker's ability to take part in more FEEL communication rounds; therefore, leading to a slower convergence rate. 
 It should be emphasized that our work focuses on the fine-grained selection of training samples. In this work, the fine-grained selection refers to using a subset of the local data samples for updating the global model. This fine-grained selection can be clearly contrasted with the literature, which focuses on the coarse-grained selection of workers whose datasets are either fully included or fully excluded (i.e., no partial inclusion/exclusion of data samples).
 Some remarkably critical yet overlooked questions are: do all local data samples contribute equally to global model improvement? How can data be filtered to conserve energy while attaining the desired performance considering both network and device resource constraints? Does the deadline constraint help in optimizing the processing and transmission power? Are there other system parameters that can be optimized to save energy further?
 
 To this end, this paper contributes to the state-of-the-art by introducing a novel approach for model training, local computation, and communication resource allocation to support energy-efficient FEEL systems. 
In this work, we adopt a unique approach in which we explore the fine-grained selection of training data for improved energy efficiency of FEEL. Our proposed approach is motivated by the fact that not all local data samples can significantly contribute to the global model. 
 We consider the FEEL system using a practical wireless network setting in which the battery-constrained edge devices are connected to the edge server. 
 The edge server is equipped with multiple antennas that employ beamforming techniques to maximize the signal-to-noise ratio.
 Each edge device represents a worker that trains its local learning model using its local data and then sends the model parameters back to the edge server. 
 The edge server aggregates all local parameters and forms a global model. 
 Then the resulting model is broadcasted to the workers for further updates. 
 Due to the limited bandwidth, only a subset of these workers is chosen at every FEEL round to take part in the training process. To synchronize the updates and avoid long waiting times, the server employs a FEEL round deadline constraint.
 The key contributions of this work can be summarized as follows:

\begin{itemize}
  \item Formulate a joint optimization problem with the aim of minimizing the energy consumption and the optimal allocation of available resources while satisfying the learning performance and the deadline constraint, 
  
\item Introduce tractable solutions to solve this problem. 
We first propose a fine-grained data selection algorithm that leverages the global model to select only the samples that contribute to improving the model's performance before executing local training steps. 
Then, we present a mathematical proof that supports the intuition behind the proposed algorithm and show the fundamental logic behind excluding the data samples predicted with high probability. 

\item Utilize the deadline constraint to give each worker more flexibility to reduce energy consumption further. 
Each worker exploits the waiting time as an opportunity to reduce the computation and transmission power rather than sending the update immediately once the training is completed, 
    
\item  Derive closed-form solution followed by a Golden-Section based iterative algorithm to find the optimal solution for optimal beam vector, allocated bandwidth, local CPU speed, and transmission power that minimize the total energy consumption,  
    
\item Carry out extensive simulation experiments using MNIST and CIFAR-10 datasets under independent and identically distributed (i.i.d) and non-i.i.d. data distribution assumptions to empirically verify the theoretical analysis. 
Our experiments demonstrate that the proposed technique can substantially reduce the local energy consumption compared to the baseline FEEL algorithm while achieving similar accuracy. 

\end{itemize}

The remainder of the paper is organized as follows. Related literature is presented in Section~\ref{relatedwork}. Next, Section~\ref{sysmodel} presents the system model. 
Subsequently, we formulate the problem statement in Section ~\ref{sec:Problemformulation} while the proposed approach supported by mathematical proof is given in Section~\ref{sec:proposedsol}. 
Section~\ref{sec:experiment} presents the experimental setup,  performance evaluation results and discussion, and the lessons learned.
Finally, we conclude our work and provide future research directions in~\ref{conclusion}.

\section{RELATED WORK}
\label{relatedwork}
The use of FL over wireless networks (i.e., FEEL) has received considerable interests in the literature ~\cite{sattler2019robust,lin2017deep,wang2019adaptive,zhu2019broadband,nishio2019client,amiri2019machine,yang2018federated,liu2019edge}. 
The work in \cite{naderializadeh2020communication} investigated the transmission delay for decentralized learning on the wireless channels, where every client is authorized to connect to its neighbors. 
In \cite{prof_nguyen_fl}, the authors proposed an optimization model for joint energy consumption and completion time on FEEL, considering the power allocation, local resources, and model performance.  
The work in \cite{wadu2020federated} studied the channel uncertainty where an optimization problem is formulated to minimize the loss function while considering the scheduling process and resource allocation. 
The authors in \cite{chen2020joint,chen2020convergence} investigated the joint optimization of model training and resource allocation. 
However, energy consumption was not considered.
\begin{table*}[t]
\small
    \centering
    \caption{\uppercase{Relationship between our work and the recent literature}}
    \begin{tabular}{|c|p{2cm}|p{2cm}|p{2cm}|p{2cm}|p{2cm}|} \hline
        \textbf{Ref} &  \textbf{FEEL Round Deadline} & \textbf{Completion Time} &  \textbf{Devices Heterogeneity}  &  \textbf{Energy Budget} & \textbf{Data Exclusion}   \\ \hline
        \cite{mo2020energy}  
        & \xmarklight  & \checkmark  & \checkmark & \xmarklight & \xmarklight \\
        \cite{zeng2020energy}  
        & \checkmark  & \checkmark & \checkmark & \xmarklight & \xmarklight \\ 
        \cite{wang2020federated}  
        & \xmarklight  & \checkmark & \checkmark & \xmarklight & \xmarklight \\ 
     \cite{yang2020energy}  
        & \xmarklight  & \checkmark  & \checkmark & \xmarklight & \xmarklight \\ 
      \cite{luo2020hfel}  
        & \xmarklight  & \checkmark  & \checkmark & \xmarklight & \xmarklight \\ 
              Our work 
        & \checkmark  & \checkmark   & \checkmark & \checkmark & \checkmark  \\ \hline
    \end{tabular}
    \label{tab:related}
\end{table*}

Focusing on energy constraints, Zeng \textit{et al.} \cite{zeng2020energy} considered the energy-efficiency of FEEL where the goal is to minimize the total energy consumption. 
A minimization problem is formulated, and a greedy allocation algorithm is proposed to allocate more resources to weak devices.
However, the local computation energy consumption is not considered as it is assumed to be fixed and uniform among all devices, which is impractical as the data is imbalanced, and the processing capabilities are varied.  In ~\cite{wang2020federated}, Wang \textit{et al.} studied  {energy-efficiency} for FEEL where the computation and communication resources are considered. 
An optimization problem is formulated to minimize the completion time as well as computation and transmission energy. 
Furthermore, the authors in ~\cite{mo2020energy} proposed an approach to minimize the total energy consumption across all workers during predefined training time. 

It is worth noting that some prior works \cite{zeng2020energy,wang2020federated,yang2020energy,mo2020energy,luo2020hfel} introduced their approaches in a coarse-grained level, assuming that all data samples of any selected client are used in every local iteration. 
The stragglers (i.e., the devices with bad channels, low CPU speed, or insufficient energy) aren't considered, which might delay the whole training cycle. 
Specifically, the deadline of FEEL communication rounds is not considered where the server has to wait before fusing the updates and starting a new FEEL training round. 
Often, the whole completion time is considered while overlooking the effects of data and device heterogeneity. Heterogeneity can lead to scenarios in which some devices finish their task earlier while others may stay ideal for a long time waiting for the fused model to start a new FEEL round. 
Also, the energy budget is not considered, which is essential to complete the updates. 

Although there is a lot of research devoted to studying energy-efficient FEEL \cite{zeng2020energy,wang2020federated,yang2020energy,mo2020energy,luo2020hfel}, our work differs from these works as shown in Table \ref{tab:related}. 
Specifically, we adopt a unique approach that allows for the fine-grained selection of training data for improved energy efficiency of FEEL. Our proposed  approach is motivated by the fact that not all local data samples can significantly contribute to the global model~\cite{Albaseer2021}.  {Throughout} this work, we jointly consider the learning algorithm and the corresponding imbalanced local data samples as well as resource and system constraints.  

\section{SYSTEM MODEL}
\label{sysmodel}
\begin{figure}[t]
\centering
  \includegraphics[width=\linewidth,height=7cm]{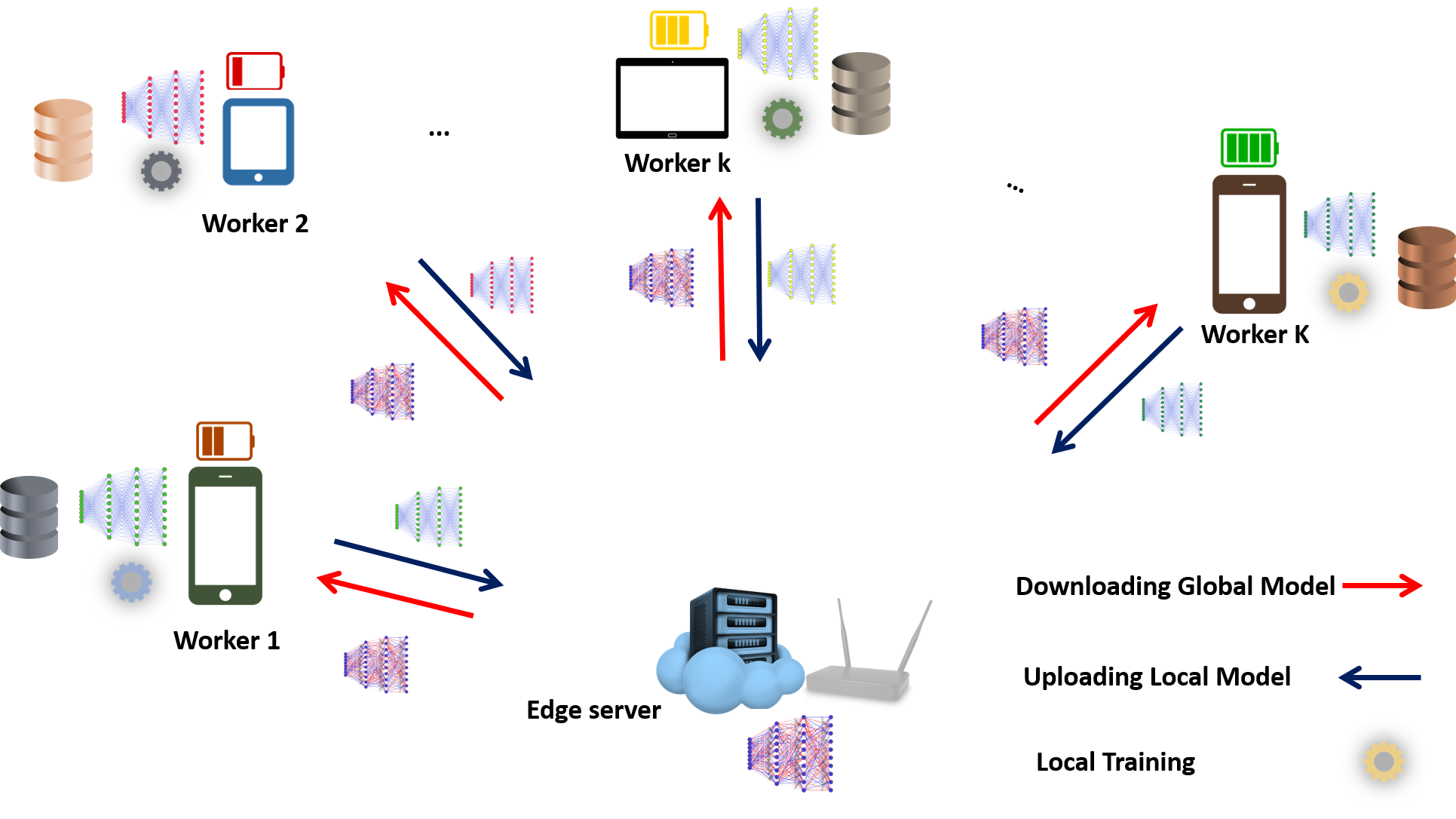}
\caption{ {FEEL where a $K$ battery-constrained edge devices are connected to edge server over RF Access point.}}
\label{fig:sysmodel}
\end{figure}
The considered FEEL system consists of a set of end devices $\mathbf{K}$ that are connected to the edge server as shown in~\Cref{fig:sysmodel}. 
The edge server is located in the contiguity of $\mathbf{K}$ workers to organize and coordinate the training process.  
Each worker $k \in \mathbf{K}$ uses its own data $\mathcal{D}_k$  to train its model $\mathcal{\theta}_k$ locally and then sends the updates (i.e., weights and biases) back to the server, where $\mathcal{D}_k=\{{x}_{k,d}\in \mathbb{R}^d, y_{k,d} \in\mathbb{R}\}$, and $|\mathcal{D}_k|$ is the portion of data samples and the whole data among workers is
$D\triangleq\sum_{k=1}^{K}|\mathcal{D}_{k}|$ where $K = |\mathbf{K}|$. 
Here, ${x}_{k,d}$ is the $d$-dimensional input data vector at the $k$-\textit{th} worker, and $y_{k,d}$ is the corresponding label associated with ${x}_{k,d}$.  
In return, the server collects and fuses all the workers' updates to build a global model. 
In each FEEL round, the server sets a deadline constraint to synchronize the updates and avoid long waiting times. 
Initially, the server sends random parameters $\mathcal{\theta}_0$ to all selected workers to start the training. On the worker side, the received global model is used as a reference to train the local models and control the divergence. 
All these steps incur massive energy consumption; thus, for each worker to join the learning process, local data samples, energy budget, FEEL round deadline, and computation and communication capabilities are considered to ensure a robust FEEL system and avoid losing selected worker updates due to insufficient energy or FEEL round deadline time out. 
To facilitate the presentation, we summarize the utilized main symbols in Table \ref{Tab:Notation}. 
\begin{table}[h]
\centering
\scriptsize
\caption{LIST OF IMPORTANT NOTATIONS}
\label{Tab:Notation}
\begin{tabular}{p{1cm} p{4.7cm}} 
 \hline
$\mathbf{K}$	& a set of collocated edge devices \\
$k$	& worker $k$ where $k \in \mathbf{K}$\\
$\mathcal{D}_k$ & the local data held by $k$-\textit{th} worker\\
$\mathcal{\theta}_r$ & model parameters at $r$-\textit{th} FEEL round \\
$F_r(\mathbf{\theta})$ & the global loss function at $r$-\textit{th} FEEL round \\
$f_{s}$ & the loss function that captures the error of each local data sample \\
$\theta_k$ & the local model parameters of the $k$-\textit{th} worker \\
$\mathbf{N}$ & number of local updates \\
$\eta$ & learning rate \\
$\varepsilon$ & number of epochs \\
$b$ & batch size \\
$T^{cmp}_k$ & local computation time of $k$-\textit{th} worker\\
$f^\mathrm{cmp}_k$ & the used CPU frequency at $k$-\textit{th} worker device\\
$\Phi$ & number of cycles required  to process one sample \\

$\textbf T$ & the FEEL round deadline constraint set by the server at every $r$-\textit{th} FEEL round \\
$T^\mathrm{ up}_k$ & the required time to upload the update to the server\\
$E_k^{cmp}$ & local energy consumption for every $k$-\textit{th} worker\\
$E^\mathrm{ up}_k $ & transmission energy consumption of  the $k$-{th} worker \\
$R_{k}^{up}$ & uplink data rate achieved by the $k$-{th} worker\\
${\mathbf h}_k$ & the uplink channel gain between the $k$-{th} worker and the $M$-antenna BS\\
${\mathbf w}_k$ & the $k$-{th} worker beamforming vectors received from $M$-antenna BS\\
$P^{ up}_k$ & the $k$-\textit{th} worker transmit power \\
${E}_{k}$ & the energy budget at $k$-\textit{th} worker\\
$f^\mathrm{ max}_k$ & maximum CPU frequency at $k$-\textit{th} worker\\
$f^\mathrm{ min}_k$ & minimum CPU frequency at $k$-\textit{th} worker\\
$\xi$ & model size \\
$P^\mathrm{max}_k$ & the maximum transmit power\\
$P^\mathrm{min}_k$ & the minimum transmit power\\
\hline
\end{tabular}
\end{table}

Before setting up and defining our problem,  we present an overview of the learning, computation, communication, and energy consumption models utilized in this work. 
\subsection{Feel Model}
\label{sec:loss}
The local loss function captures the performance of the model on a given dataset $\{{x}_{k,d},{y}_{k,d}\}$ for the $k$-\textit{th} worker at the $r$-\textit{th} FEEL round, and total loss over all data samples is defined as follows:
\begin{equation}
F^k_{r}(\mathbf{\theta_k}) \triangleq \frac{1}{\left|\mathcal{D}_{k}\right|}\sum_{s\in\mathcal{D}_{k}}f_{s}(\mathbf{\theta_k}).
\label{eq:localLossFuncAllSamples}
\end{equation}

where $f_{s}$ captures the error of each local data sample and $\theta_k$ is the local model parameters.

To train its local model, the $k$-\textit{th} worker runs its local solver, such as mini-batch stochastic gradient descent (SGD), locally to minimize the loss function defined in Eq. \eqref{eq:localLossFuncAllSamples} for several local epochs denoted by $\varepsilon$. 
Specifically, the local model parameters  $\theta_k$ are updated as follows:
    \begin{equation}
    \label{eq:local}
\mathbf{\theta^{(k)}_{n}}=\theta^{(k)}_{n-1}-\eta \nabla F^k_{r}(\theta_n^{(k)})
\end{equation}
 where $\eta$ is the step size (i.e., learning rate) at each FEEL round, $n = 1,2,\dots, \mathbf{N}$ local update index performed by the $k$-\textit{th} worker as: $\mathbf{N} = \varepsilon \frac{|\mathcal{D}_{k}|}{b}$ where $b$ is the batch size and $\varepsilon$ is the number of epochs.  {In \eqref{eq:local}, $\theta^{(k)}_{0}$ denotes the global parameters received from the server and $\theta^{(k)}_{\mathbf{N}}$ denotes the last local updated parameters by $k$-\textit{th} worker which will be sent back to the server after $\mathbf{N}$ local iterations. 
 In the rest of  {the paper}, we use $\theta^{(k)}_{\mathbf{N}}$ as $\theta^{(k)}_{r}$ to simplify the exposition.}

After uploading all local updates to the server, the global loss function at every $r$-\textit{th} FEEL round is defined as:
\begin{equation}
F_r(\mathbf{\theta}) \triangleq {\sum_{k=1}^{K}\delta_k F^k_{r}(\mathbf{\theta})}.
\label{eq:globalLossFuncAllSamples}
\end{equation}
where the local data samples $\delta_k$ is weighted as follows:
 \begin{equation}
 \label{weightedLocalData}
     \delta_k = \frac{ {|\mathcal{D}_{k}|}}{D}.   
 \end{equation} 
 
Accordingly, the global model parameters are computed as follow:
\begin{equation}
\mathbf{\theta^r} = {\sum_{k=1}^{K}\delta_k \theta^{(k)}_{r}}\label{eq:globalAverage}.
\end{equation}
$F_r(\mathbf{\theta})$ and $\mathbf{\theta_r}$ are sent to all selected workers in the ($r+1$)-\textit{th} FEEL round to train and update the model parameters. 
Thus, the aim is to find the global parameters $\mathbf{\theta}^{*}$ that minimize $F(\mathbf{\theta})$.
\begin{equation}
\mathbf{\theta}^{*} \triangleq \arg\min F(\mathbf{\theta}).
\label{eq:learningProblem}
\end{equation}

\subsection{Local Computation and Energy Models}
To train local models, each $k$-{th} worker partitions its local data $D_k$ into batches of size $b$ and trains its local model for a number of epochs $\varepsilon$. 
Thus, the local computation delay $T^{cmp}_k$ can be defined as:
\begin{align}
\label{eq:localtime2}
T^{cmp}_k =\varepsilon \frac{ {|\mathcal{D}_{k}}| \Phi}{f^\mathrm{cmp}_k}
\end{align} 
where $f^\mathrm{cmp}_k$ denotes the local processing speed (i.e., CPU frequency), and $\Phi$ is the number of cycles to handle one sample. 
To finish the local training, every $k$-\textit{th} worker consumes $E_k^{cmp}$ energy defined as\cite{mao2016dynamic}:
\begin{equation} 
    \label{eq:LocalEnergy}
   E_k^{cmp} = \frac{\alpha_k}{2} (f^\mathrm{cmp}_k)^3 T^{cmp}_k
\end{equation}
where $\frac{\alpha_k}{2}$ is the energy capacitance coefficient of a given device. 
Substituting \eqref{eq:localtime2} into the right hand-side of \eqref{eq:LocalEnergy} yields: 
\begin{align}
    \label{eq:LocalEnergy2}
   E_k^{cmp} = \frac{\alpha_k}{2} (\varepsilon (f^\mathrm{cmp}_k)^2  { {|\mathcal{D}_{k}|} \Phi})  
\end{align}
Specifically, after finishing the local model updates, each worker uploads its model to the edge server and waits for the fused global model to start a new FEEL round, as shown in \Cref{fig:prob_s}. 
This time can be exploited to conserve energy, as explained next.
\begin{figure}[t]
\centering
  \includegraphics[width=0.8\linewidth]{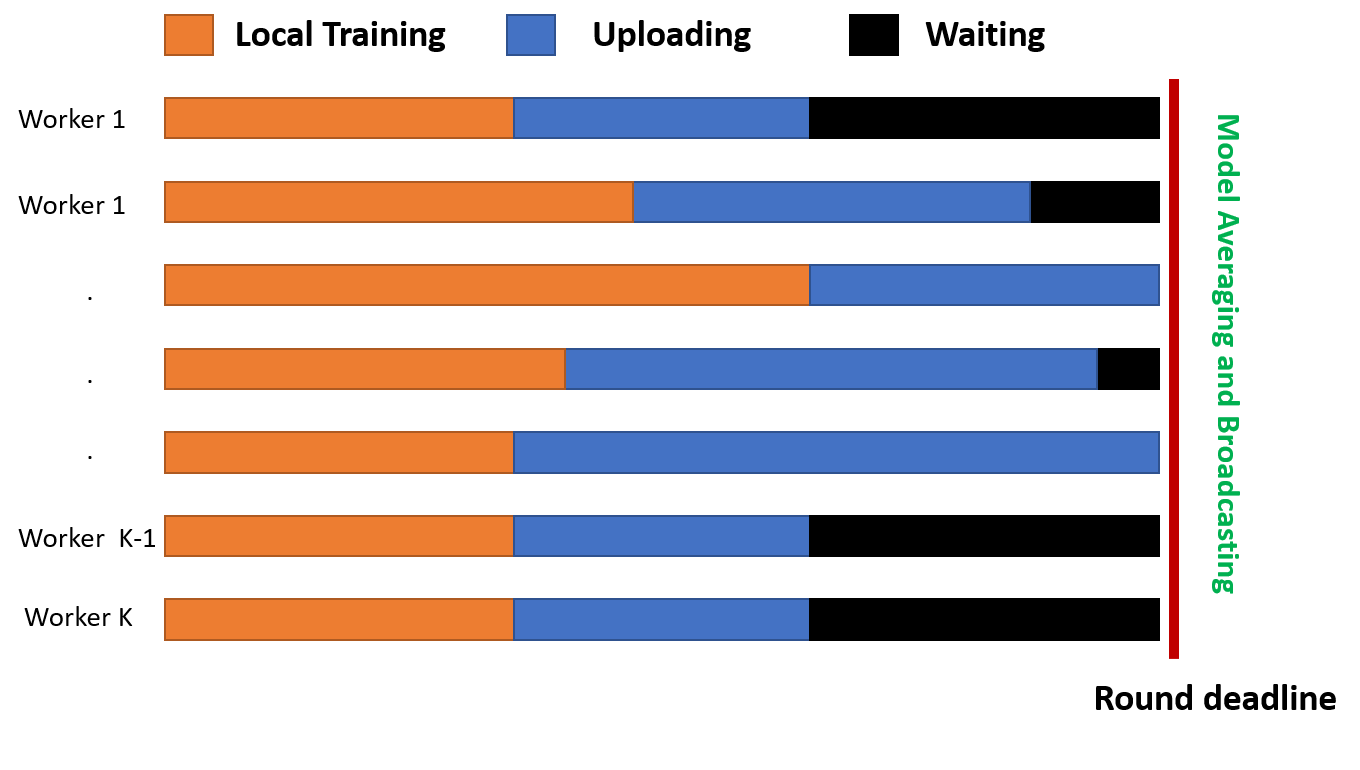}
\caption{FEEL Round over Wireless Channel (OFDMA) under data and resource heterogeneity.}
\label{fig:prob_s}
\end{figure}

\subsection{Transmission Delay and Energy Models}
For the communication model, we assume that the server is equipped with multiple antennas employing beamforming techniques to communicate with the local users through orthogonal frequency-division multiple access (OFDMA) channels with a total bandwidth $B$. 
Each k-\textit{th} worker is assigned a bandwidth $\lambda_k B$ to upload its update where $\lambda_k$ is the allocation ratio, $0 \le \lambda_k \le 1$. 
We denote the uplink channel gain between the $k$-{th} worker and the $M$-antenna edge server by ${\mathbf h}_k\in \mathbb C ^{M}$.
Subsequently, the achievable data rate for every $k$-{th} worker is defined as:

\begin{footnotesize}
\begin{align}\label{eq:achievedrate}
{R_{k}^{up} = \lambda_k B~\text {log}_2\left(1 + \frac{ \left| {  {\mathbf h}_k^H {\mathbf w}_k} \right|^2 P^{ up}_k}{ { \lambda_k B \mathbf w}_k^H \left( \sum\limits_{k' \ne k } {\mathbf h}_{k'} {\mathbf h}_{k'}^H +  \sigma^2_{0}  {\mathbf I} \right) {\mathbf w}_k }\right)},
\end{align} 
\end{footnotesize}
where ${\mathbf w}_k \in \mathbb C ^{M}$ denotes the beamforming weights, $P^{ up}_k$ is the $k$-\textit{th} worker transmission power, (.)$^H$ stands for the Hermitian operation, $\sigma^2_{0}$ is the spectral density power of the background noise, and ${\mathbf I}$ is the identity matrix. 
Accordingly, the uploading delay can be defined as:
\begin{align} \label{eq:uploadlatency}
 T^\mathrm{ up}_k = \frac{\xi}{R_{k}^{up}}
\end{align}
where $\xi$ denotes size of the model parameters. Further, the associated consumed energy is expressed as: 
\begin{align}\label{eq:uploading}
E^\mathrm{ up}_{k} = T^\mathrm{up}_k P^{ up}_k.
\end{align}
\setlength{\textfloatsep}{0pt}

 {In reality, the edge server needs to employ a FEEL round deadline constraint $\textbf T$ at every FEEL training round to synchronize the updates and avoid lengthy waiting times, especially for stragglers, to start a new global training round. 
Therefore, each $k$-{th} worker has to complete its computation and communication phases within $\textbf T$.} 
Formally, the computation and communication time for each worker has to satisfy this condition:
\begin{equation}
\label{eq:client_delay}
 T^\mathrm{ cmp}_k +  T^\mathrm{ up}_k \le \textbf T
\end{equation}
From \Cref{fig:prob_s}, we can note that some workers have to idly wait for the most recent combined model to start the next FEEL training round even if they finish the training and uploading tasks before the FEEL round deadline. 
Thus, in this work, instead of idly waiting for the fused model to be returned back from the server, we utilize this time as an opportunity for lowing the expended energy during the computation and communication phases by lowering the CPU speed and transmission power, respectively. To achieve this,  \eqref{eq:client_delay} is redefined as:
\begin{equation}
\label{eq:client_synch}
 T^\mathrm{ cmp}_k +  T^\mathrm{ up}_k = \textbf T
\end{equation}
\section{PROBLEM FORMULATION}
\label{sec:Problemformulation}

Given the above system model and discussions, we aim to minimize the total expended energy among workers during FEEL rounds, subject to constraints on learning performance, the computation and communication resource constraints, and the FEEL round deadline. 
We can formulate the optimization problem as follows:

\begin{subequations}
\footnotesize
\label{eq:OptmizedProblem1}
\begin{align}
	\textbf{$P_1$:} \quad   \underset{\mathcal{D}_k, P^{ up}_k, T^\mathrm{ up}_k, T^\mathrm{ cmp}_k,  \atop f^\mathrm{cmp}_k, \lambda_k, {\mathbf w}_k} \min \quad & 
\SumNoLim{r=1}{R} \SumNoLim{k}{K} \mathbf{I(k)}(E_k^{cmp} + E^\mathrm{ up}_{k} )   \label{eq:OptmizedProblema} 
\\
\text{s.t.:}  \nonumber
\\
          \quad & F(\mathbf{\theta}) - F(\mathbf{\theta^*}) \le \epsilon \label{eq:convergence_constraint}
\\
       \quad & \SumNoLim{k}{K} \mathbf{I(k)} \lambda_k \leq 1 \quad \label{eq:bandwidth_constraint}
\\
			\quad &\mathbf{I(k)}(E^\mathrm{cmp}_k + E^\mathrm{ up}_k) \le  {E}_{k}, \quad (\forall k) \label{eq:OptmizedProblemEnergy}
\\
			\quad &\mathbf{I(k)}( T^\mathrm{ cmp}_k + T^\mathrm{ up}_k) = \mathbf T, \quad (\forall k) \label{eq:OptmizedProblemDeadline}
\\
			\quad & P^\mathrm{min}_k \le P^{ up}_k \le P^\mathrm{max}_k, \quad (\forall k)	\label{eq:OptmizedProblem_pwr_transmit}
\\	
			\quad & f^\mathrm{ min}_k \le f^\mathrm{cmp}_k \le f^\mathrm{ max}_k, \quad (\forall k) \label{eq:OptmizedProblemEnergyPU}
\\
	\quad & R_{k}^{up} \ge \xi, \quad (\forall k) \label{eq:OptmizedProblem_up_model}
\\
			\quad & \left| {\mathbf w}_k \right|^2 = 1, \quad (\forall k) \label{eq:OptmizedProblemBeamForming}
\\
	    \quad &  \mathbf{I(k)} \in \{0, 1\}   \quad (\forall k)\label{eq:selection_var}
\end{align}
\end{subequations}

Constraint~\eqref{eq:convergence_constraint} is set to guarantee the convergence of the global federated model assuming that $\mathbf{\theta^*}$ is the optimal targeted model obtained using a virtual centralized ML algorithm and the whole datasets $D$. 
Constraint \eqref{eq:bandwidth_constraint} specifies that the total allocated bandwidth can not exceed the total allocated system bandwidth.
Constraint \eqref{eq:OptmizedProblemEnergy} ensures that the energy expended for computation and communication does not surpass the energy budget of any $k$-\textit{th} worker. 
This constraint ensures that the selected worker has sufficient energy to avoid losing the update.
The constraint \eqref{eq:OptmizedProblemDeadline} ensures that the total computation and upload time are restricted to the FEEL round deadline ${\mathbf T}$ to avoid longer waiting time. 
It is worth noting that in our work all selected workers have the same finishing time. This provides more flexibility when optimizing the CPU frequency and transmission power.
The transmission power of every selected worker is restricted in \eqref{eq:OptmizedProblem_pwr_transmit} to be between the minimum $P^\mathrm{min}_k$ and the maximum transmission power $ P^\mathrm{max}_k$. 
Constraint ~\eqref{eq:OptmizedProblemEnergyPU} ensures that the CPU-frequency of the $k$-\textit{th} worker ranges between the minimum $f^\mathrm{ min}_k$ and maximum $f^\mathrm{ max}_k$ CPU frequencies.
Constraint \eqref{eq:OptmizedProblem_up_model} ensures that the achievable upload rate of each $k$-\textit{th} worker is sufficient to send the model (i.e., the updated parameters) to the server.
 Constraint \eqref{eq:OptmizedProblemBeamForming} ensures that the beamforming vectors do not increase the total transmission power. 
 Last, constraint \eqref{eq:selection_var} is an indicator function that specifies whether the k-\textit{th} worker is selected $\mathbf{I(k)} =1$ in the $r$-\textit{th} FEEL round or not $\mathbf{I(k)}=0$.

Unfortunately, \textbf{$P_1$} is intractable and hard to solve as it requires future offline information such as energy budget level, channel states, and CPU speed for all participating workers. 
Such information is very challenging to be accurately predicted due to other running processes, dynamic channels, and availability (i.e., the device might be switched off or not connected to the server). 
Besides, constraint \eqref{eq:convergence_constraint}, which requires the optimal model parameters, is impractical under FEEL assumptions as the data is kept locally and can't be accessed by the server. 

\section{PROPOSED APPROACHES}
\label{sec:proposedsol}
To solve \textbf{$P_1$}, we first choose a fixed number of FEEL global rounds $R$ to be large enough while satisfying the desired accuracy. 
It is worth noting that it is difficult to find a closed-form that determines the correlation between the number of FEEL rounds and convergence in non-convex learning tasks (i.e., deep neural networks); thus, iterative training updates are used until converge. 
We then reformulate \textbf{$P_1$:} as an online optimization problem at every $r$-\textit{th} FEEL round, a subset of available clients can join the learning process, and the global model periodically evaluated. 
Formally, the reformulated problem is defined as: 
\begin{subequations}
\label{eq:OptmizedProblem2}
\begin{align}
	\textbf{$P_2$:} \quad   \underset{\mathcal{D}_k, P^{ up}_k, T^\mathrm{ up}_k, T^\mathrm{ cmp}_k,  \atop f^\mathrm{cmp}_k, \lambda_k, {\mathbf w}_k} \min \quad & 
 \SumNoLim{k}{K} \mathbf{I(k)}(E_k^{cmp} + E^\mathrm{ up}_{k} )   \label{eq:OptmizedProblema2} 
\\
\text{s.t.:}  \nonumber
\\
  & \eqref{eq:bandwidth_constraint}- \eqref{eq:selection_var} \nonumber
\end{align}
\end{subequations}
We can notice that \textbf{$P_2$} is still hard to solve due to the combinatorial nature of \textbf{$P_2$} which has high complexity search space over the selected workers. 
Also, variables $P^{ up}_k$, $T^\mathrm{ up}_k$, and $T^\mathrm{ cmp}_k $ are all coupled in constraints \eqref{eq:OptmizedProblemEnergy}, \eqref{eq:OptmizedProblemDeadline}, and \eqref{eq:OptmizedProblem_up_model}. 
This indicates that, it is impossible to reach the direct optimal solution for this problem. 
Thus, efficient tractable solutions with low-complexity are highly desired, and this motivates the design of the proposed algorithms as detailed in the following subsections.

First, we propose a novel local training algorithm that excludes less relevant data samples and effectively chooses the samples that improve the model's performance followed by a mathematical proof in Sections \ref{sec:proposedtrainingalg}, \ref{Theortical_proof}. 
This Algorithm enables the participating workers to select the optimal data samples $\mathcal{D}_k$ that reduce the computation time and conserve local energy consumption; therefore, leading to a further decrease in expended energy. 
Then, we find the optimal value for beamforming weights ${\mathbf w}_k$ and the allocated bandwidth $\lambda_k$, which in turn maximizes $R_{k}^{up}$ and leading to minimize the upload energy.
After that, in Section~\ref{iterativealgorithms}, we derive closed-form solutions followed by a Golden-Section based iterative algorithm to find the optimal solution for local CPU speed $f^\mathrm{cmp}_k$, and transmission power $P^{ up}_k$ to minimize energy consumption based on "reduced" local samples, the optimal value of ${\mathbf w}_k$ and $\lambda_k$. Last, we summarize the overall approach in Section~\ref{summayalg}

\subsection{Proposed Local Training Algorithm}
\label{sec:proposedtrainingalg}
In this Algorithm, all chosen workers receive  {the global model $\mathbf{\theta^{(k)}_{0}}$ from the corresponding server and utilize the whole local samples to update the received model parameters only once (i.e., initialization epoch $\varepsilon = 1$) to specialize the global parameters and reduce the divergence between the global and local models. 
Subsequently, all selected workers use the updated model in the first epoch to determine the local samples that need to be included or excluded. 
To this end, each local sample is fed into the model $\mathbf{\theta^{(k)}_{1}}$, which produces different probabilities based on a given number of classes. 
The maximum probability is compared to a predetermined threshold probability $\vartheta$. 
This threshold stipulates below which samples are included in later epochs while the samples predicted with a probability greater than threshold $\vartheta$ are excluded. }
Formally, this can be defined as follows:
\begin{equation}
\left\{
\begin{array}{l}
if \max\{p({x}_{d},\mathbf{\theta^{(k)}_{1}})\} > \vartheta \quad Exclude\\
if \max\{p({x}_{d},\mathbf{\theta^{(k)}_{1}})\} \le \vartheta \quad Include
\end{array}
\right.
\label{eq1}
\end{equation}
The motivation behind this algorithm stems from the fact that samples predicted with high probability do not contribute much to the loss function. 
The mathematical proof is given in Section \ref{Theortical_proof}. Given the number of local samples $|\mathcal{D}_{k}|$, we denote the number of excluded samples by $\kappa$, and the number of included samples in later epochs by $|\mathcal{D}_{k}|- \kappa$. 
Hence, the updated computation time needed to complete the local updating task can be redefined as:

\begin{align}
\label{eq:localtimeproposed}
T^{cmp}_k 
= \frac{(\varepsilon \phi |\mathcal{D}_{k}|) - \kappa (\varepsilon-1)}{f^\mathrm{cmp}_k}
\end{align}
Accordingly, the corresponding local energy computation is rewritten as:
\begin{align}
    \label{eq:LocalEnergyProposed}
   E_k^{cmp} = \frac{\alpha_k}{2} (\varepsilon-1) (f^\mathrm{cmp}_k)^2  {(|\mathcal{D}_{k}|-\kappa) \Phi}) +   \frac{\alpha_k}{2} (f^\mathrm{cmp}_k)^2  {|\mathcal{D}_{k}| \Phi}) 
\end{align}
The steps of this algorithm are summarized in Algorithm~\ref{alg:Proposedalg2}. 
In Algorithm~\ref{alg:Proposedalg2} step 1, each Worker $k$ updates the received global model $\mathbf{\theta^{(k)}_{0}}$ for one epoch. 
In steps 2-5, each worker uses the updated model $\mathbf{\theta^{(k)}_{1}}$ to filter all local samples and append only the ones that contribute most to the loss function. 
The appended samples with $P({x}_{d}, \mathbf{\theta^{(k)}_{1}}) \leq \vartheta$ are used to train the local model in the remaining epochs, steps 5 and 7.  
\begin{algorithm} [t]
\footnotesize
\setstretch{1.35}
    \caption{Local Training} \label{alg:Proposedalg2}
    \textbf{Local Updating:} Each worker $k$ updates $\mathbf{\theta^{(k)}_{0}}$ for one epoch\;
    \textbf{Local Predicting:} Each worker $k$ utilizes the model $\mathbf{\theta^{(k)}_{1}}$ updated in the first epoch to filter all local samples;
     \textbf{Set $\mathcal{D}_k^{r} = \{ \}$}\;
     \For{$d = 1$ to $|\mathcal{D}_{k}|$}{
     \If{$P({x}_{d}, \mathbf{\theta^{(k)}_{1}}) \leq \vartheta$}  
     { $\mathcal{D}_k^{r} = \mathcal{D}_k^{r} \cup\{{x}_{d} , {y}_{d}$\}
     }
     }
      \For{$epoch=2$ to $\varepsilon$}{
      Each worker $k$ continues the training task using only $\mathcal{D}_k^{r}$
      }
\end{algorithm}

\subsection{Mathematical Analysis}
\label{Theortical_proof}
This section presents mathematical proof that provides the theoretical foundation for excluding the data samples predicted with high probability. 
As we deal with a classification problem, we use the commonly-used cross-entropy loss function that computes the difference between the ground truth and the prediction as follows:
\begin{align}
\label{eq:cross_entropy}
 F^k_{r}(\mathbf{\theta})=\frac{1}{\left|\mathcal{D}_{k}\right|}\sum_{s\in\mathcal{D}_{k}} \mathbb{E}_{{\x, {y}_{d} \sim p}}[\sum_{c=1}^{C} p({{y}_{d}=c}) \log f_c(\x, \theta)]  \nonumber \\
= \frac{1}{\left|\mathcal{D}_{k}\right|}\sum_{s\in\mathcal{D}_{k}} \sum_{c=1}^{C} p({y}_{d}=c) \mathbb{E}_{\x|{y}_{d}=c}[\log f_c(\x, \theta)].
\end{align}
where $[C] = \{1, \ldots, C\}$ denotes the number of classes, and $p$ is the class probability distribution. 
We should note that we ignore the negative sign at the beginning of the formula, as in this analysis, we are only interested in the shape of the function. 
In general, each worker aims to solve the following learning problem:

\begin{equation}
 \arg\min F^k_{r}(\mathbf{\theta}).
\label{eq:learningProblem_worker}
\end{equation}
By substituting~(\ref{eq:cross_entropy}) into \eqref{eq:learningProblem_worker}, yields:
\begin{align}
\label{eq:local_cross_entropy}
  \min_{\theta_k} \frac{1}{\left|\mathcal{D}_{k}\right|}\sum_{s\in\mathcal{D}_{k}}\sum_{c=1}^{C} p({y}_{d}=c) \mathbb{E}_{\x|{y}_{d}=c}[\log f_c(\x, \theta)].
\end{align}

To find the optimal $\theta$, each k-\textit{th} worker  {uses mini-batch SGD as a local solver to iterativly solve~\eqref{eq:local_cross_entropy} as it converges directly to minima when the dataset is small}. Then, based on \eqref{eq:local}, the following updates is performed:

\begin{footnotesize}
\begin{align}
\label{eq:localUpdate1}
\theta^{(k)}_{n} =\theta^{(k)}_{n-1} - \eta \frac{1}{\left|\mathcal{D}_{k}\right|}\sum_{s\in\mathcal{D}_{k}}\sum_{c=1}^{C}  p^{(k)}({y}_{d}=c) \nabla_{\theta} \mathbb{E}[\log f_c(\x, \theta^{(k)}_{n-1})]
\end{align}
\end{footnotesize}

Where we write $\mathbb{E}_{\x|y=c}$ as $\mathbb{E}$ for brevity. To simplify the expression for the analysis, we consider only one data sample.
Therefore, the updated local model parameters through $\mathbf{N}$ local iterations every r-\textit{th} FEEL round can be defined as: 
\begin{equation}
\label{local:updates2}
  \theta^{(k)}_{r} = \theta  
^{r-1} - \eta \sum_{n=1}^{\mathbf{N}} \sum_{c=1}^{C}  p^{(k)}({y}_{d}=c) \nabla_{\theta} \mathbb{E}[\log f_c(\x, \theta^{(k)}_{n-1})],\end{equation}
In~\eqref{eq:localUpdate1} and \eqref{local:updates2}, $\theta^{(k)}_{0}$ is the global model parameters received by k-\textit{th} worker from the server and $\theta^{(k)}_{\mathbf{N}}$ is the updated model parameters sent by k-\textit{th} worker. 
We can notice that server averages the received local models updated using~\eqref{local:updates2}, thus the average of all received models can be rewritten as: 
\begin{align}
\mathbf{\theta_r}=&\sum_{k=1}^{K}\delta_k (\theta 
^{r-1} - 
 \nonumber \\
 & \eta \sum_{n=1}^{\mathbf{N}} \sum_{c=1}^{C} p^{(k)}({y}_{d}=c)\nabla_{\theta} \mathbb{E}[\log f_c(\x, \theta^{(k)}_{n-1})])
\label{eq:globalAverage1}.
\end{align}
\begin{theorem}\label{thm3}
If the output of the soft max probability $P \approx 1$. 
For any input sample $i$, the difference between the predicted class probability $p_i$ and ground truth label becomes closer to $0$:
\begin{align}
\label{eq:them2}
\setlength\abovedisplayskip{3pt}
\setlength\belowdisplayskip{3pt}
\sum_{c=1}^{C} p^{(k)}({y}_{d}=c)\nabla_{\theta} \mathbb{E}[\log f_c(\x, \theta^{(k)}_{n-1})] \approx 0
\end{align}
\end{theorem}
\begin{proof}
See Appendix \ref{appendix:theorm1}.
\end{proof}
From~\eqref{local:updates2}, \eqref{eq:globalAverage1} and Theorem 1, we can infer that when probability  of a certain class $c$ approaches $1$, $p \approx 1$,  the impacts of its values among the whole samples becomes less significant. 
The details are provided in the appendix~\ref{appendix:theorm1}. 

\subsection{Iterative Algorithm To Complete The Solutions
of \textbf{P$_2$}}
\label{iterativealgorithms}
In this section, we complete the tractable solutions for \textbf{P$_2$}.
First, for the beamforming, we first find the optimal beam weights ${\mathbf w}_k$ that maximize the attainable data rate which in return minimize the transmission energy as follows~\cite{tran2020lightwave}:
\begin{align}
\label{eq:beam}
{\mathbf w}^{\star}_j =  \arg \underset{\left| {\mathbf w}_k \right|^2 = 1} \max R_{k}^{up} \quad (\forall k).
\end{align}
According to Rayleight-Ritz quotient \cite{Parlett,tran2020lightwave}, ${\mathbf w}^{\star}_k$ can be obtained by finding the eigenvector corresponding to the largest eigenvalue of the matrix ${\mathbf h}_k^H \left( \sum_{k' \ne k } {\mathbf h}_{k'} {\mathbf h}_{k'}^H +  \sigma^2_{0}  {\mathbf I} \right)^{-1}$.  
 Henceforth, let $\beta_k = \frac{ \left| {  {\mathbf h}_k^H {\mathbf w}_k^{\star}} \right|^2 }{ {\mathbf w}_k^{\star H} \left( \sum\limits_{k' \ne k } {\mathbf h}_{k'} {\mathbf h}_{k'}^H +  \sigma^2_{0}  {\mathbf I} \right) {\mathbf w}_k^{\star} }$. The optimal allocated bandwidth for any worker that minimizes the energy consumption is as follows~\cite[Appendix D]{yang2020energy}:
\begin{equation}
\setlength\abovedisplayskip{3pt}
\setlength\belowdisplayskip{3pt}
  \lambda_k B = \frac{\xi \rm{ln}2}{\left(\mathbf T- T^\mathrm{ cmp}_k\right)\left(W\left(-\Pi_k e^{-\Pi_k}\right)+\Pi_k\right)}
  \label{thm1eq1}
\end{equation}
where  $W(\cdot)$ denotes Lambert-W function, $\Pi_k = \frac{\xi \rm{ln}2}{\left(\mathbf T- T^\mathrm{ cmp}_k\right) P^{ up}_k \beta_k}$.

Next, from \eqref{eq:achievedrate} and by using exponent of log rule, 
\eqref{eq:OptmizedProblem_up_model} is derived as:
\begin{align}\label{eq:local_transmit_power}
P^{ up}_k  =  \lambda_k B \frac{2^{\frac{\xi}{T^\mathrm{ up}_k \lambda_k B}} -1}{\beta_k}.
\end{align}
By substituting \eqref{eq:local_transmit_power} into the right hand side of \eqref{eq:uploading}, we have:
\begin{align}
\label{eq:energyUploading2}
E^\mathrm{up}_k= T^\mathrm{ up}_k  \lambda_k B \frac{2^{\frac{\xi}{T^\mathrm{ up}_k B}}-1}{\beta_k}.
\end{align}

Further, let $\rho = (\varepsilon \phi |\mathcal{D}_{k}|) - \kappa (\varepsilon-1) $, then considering \eqref{eq:LocalEnergy} and \eqref{eq:OptmizedProblem_pwr_transmit}, the uploading time as in constraint \eqref{eq:OptmizedProblemDeadline} can be rewritten as
\begin{align}\label{eq:OptmizedProblem_pwr_transmit1}
\mathbf T - \frac{\rho}{f^\mathrm{min}_k} \leq  T^\mathrm{ up}_k  \leq  \mathbf T - \frac{\rho }{f^\mathrm{max}_k}  \quad (\forall k)
\end{align}
 From~\eqref{eq:OptmizedProblem_pwr_transmit1}, we can infer that $T^\mathrm{ up}_k$ is bounded and every k-\textit{th} worker can solve the following sub-optimization problem: 
 \begin{subequations}
\label{eq:OptmizedsubProblem2}
\begin{align}
	\textbf{Sub$-P2$:} \quad   \underset{T^\mathrm{ up}_k} \min \quad & 
    E_k^{cmp} + E^\mathrm{ up}_{k} \label{eq:OptmizedSubProblema2} 
\\
\text{s.t.: }  
\textrm{Eq.}\eqref{eq:OptmizedProblem_pwr_transmit1}\nonumber
\\
  &  \nonumber
\end{align}
\end{subequations}
  Then, the Golden-section search method ~\cite{William2007,tran2020lightwave} is employed to find the optimal value of $T^\mathrm{ up}_k$~\cite{William2007}  { as it needs fewer function calls.} 
  Subsequently, $P^{ up}_k, T^\mathrm{ cmp}_k,$ and $f^\mathrm{cmp}_k,$ are solved using their derived closed-forms. 
  All these steps are presented in Algorithm~\ref{alg:Golden}. 
  In steps 1-3,  Algorithm~\ref{alg:Golden} is initialized by defining the Golden ratio  $\varphi = \frac{3-\sqrt{5}}{2}$, $\epsilon$ and the number of iterations $\tau$. 
  The lower and upper bounds are determined by $a_0 = \mathbf T - \frac{\rho}{f^\mathrm{min}_k} $, and $b_0 =  \mathbf T - \frac{\rho}{f^\mathrm{max}_k}$ based on \eqref{eq:OptmizedProblem_pwr_transmit1}. 
  Then, in steps 4-15, the transmission time that minimizes energy consumption is found by iteratively shrinking the intervals. 
  At each iteration, steps 6-13, the updated interval is performed by either by reducing the left interval $a_{i+1} = a_i + \varphi (b_i - a_i)$, or by reducing the right interval $b_{i+1} = a_i + (1-\varphi) (b_i - a_i)$ in which the local minimum occurs. 
  We can note that steps 4-15 are repeated until a sufficient small interval is obtained. 
  In steps 16-19, the resultant solution for $T^\mathrm{ up}_k$ is found. This drives to attaining the optimal transmit power $P^{ up}_k$, computation time $T^\mathrm{cmp}_k$ and the optimal local CPU speed $f^\mathrm{cmp}_k$, as in steps 20-22. 
 It is worth noting that $f^\mathrm{cmp}_k$ is estimated as:
\begin{align}\label{eq:local_cpu_updated}
f^\mathrm{cmp}_k = \frac{\rho}{T^\mathrm{cmp}_k}
\end{align}

\begin{algorithm}[t]
\begin{footnotesize}
\caption{Energy Minimization}
\label{alg:Golden}
 \KwIn{$\rho = (\varepsilon \phi |\mathcal{D}_{k}|) - \kappa (\varepsilon-1) $, $\mathbf T$, ${f^\mathrm{max}_k}$, and ${f^\mathrm{min}_k}$.}
 \KwOut{$T^\mathrm{ up}_k$, $P^\mathrm{ up}_k$,  $T^\mathrm{cmp}_k$, and ${f^\mathrm{cmp}_k}$}
  \textbf{Initialize} $\varphi = \frac{3-\sqrt{5}}{2}$, $a_0 = \mathbf T - \frac{\rho}{f^\mathrm{min}_k} $, $ b_0 =  \mathbf T - \frac{\rho}{f^\mathrm{max}_k}$, $\epsilon = 10^{-6}$, $\tau= 1000$, $ i= 0$, $ T^{up}_{temp2} = a_0 +(1-\varphi)*(b_0-a_0) $,
        $T^{up}_{temp2} = a_0 +\varphi *(b_0-a_0)$\;
        Compute  $E(T^{up}_{temp1}) $using \eqref{eq:energyUploading2}  \;         
        Compute $ E(T^{up}_{temp2})$ using \eqref{eq:energyUploading2}\;
        \While{ ($(|b_i-a_i|)>\epsilon$) $\&$ ($i< \tau$))}{
        $ i=i+1$\;
            \eIf{($E(T^{up}_{temp1}) < E(T^{up}_{temp2}))$}{
                $b_i=T^{up}_{temp2}$ \;
              $  T^{up}_{temp2} = T^{up}_{temp1}$\;
                $T^{up}_{temp1}= a_i+(1-\varphi)*(b_i-a_i)$\;
 }
            {
                $a_i=T^{up}_{temp1}$ \;
                $T^{up}_{temp1}=T^{up}_{temp2}$ \;
                $T^{up}_{temp2}=a_i+\varphi *(b_i-a_i)$\;
                
}   
                Compute  $E(T^{up}_{temp1})$ using \eqref{eq:energyUploading2}  \;
                Compute   $E(T^{up}_{temp2})$ using \eqref{eq:energyUploading2}\;     

      }  
        \eIf{($E(T^{up}_{temp1}$) $<$ $E(T^{up}_{temp2})$)}{
            $T^\mathrm{ up}_k = T^{up}_{temp1}$
}
      {
          $  T^\mathrm{ up}_k = T^{up}_{temp2}$
}
Find $P^{ up}_k$ using \eqref{eq:local_transmit_power}  \;
$T^\mathrm{cmp}_k = \mathbf T - T^\mathrm{ up}_k$\;
Find $f^\mathrm{cmp}_k$ using \eqref{eq:local_cpu_updated}\;
\end{footnotesize}
\end{algorithm}

\subsection{Proposed Energy-Efficient FEEL Approach}
\label{summayalg}
This section combines the comprised algorithms and describes the FEEL training algorithm performed by the server and workers as summarized in Algorithm~\ref{alg:synch_learning}. 
In step 1 of Algorithm~\ref{alg:synch_learning} the server initiates the global model parameters, learning rate, and the number of local iterations. 
It also determines the threshold probability that all selected workers use to choose the included samples in the $r$-\textit{th} FEEL global round. 
In step 2, the server collects prior information such as data size, battery level, channel state from possible clients willing to participate.
In steps 4-5 of Algorithm~\ref{alg:synch_learning}, the server specifies the FEEL round deadline, selects the workers, and broadcasts the global model parameters for local updates. 
In step 7, all selected workers receive the global parameters from the server. 
In step 8, each $k$-\textit{th} worker updates the global model for only one epoch by invoking steps 1-5 in the proposed local training Algorithm (Alg. \ref{alg:Proposedalg2}) to utilize only the samples that have more impacts on the updates. 
In step 9, each client finds the optimal beamforming weight by solving ~\eqref{eq:beam}. 
As a consequence, in step 10, each worker runs Algorithm\ref{alg:Golden} to find computation and communication time and the associated local processing speed and transmission power that conserves the energy consumption. 
In step 11, each $k$-\textit{th} worker invokes steps 6-7 in the proposed local training Algorithm (i.e., Alg. \ref{alg:Proposedalg2}) for the rest of epochs then uploads its update to the server, step 12. 
Last, in step 13, the server aggregates all updates to form the global model. Steps (4-13) are repeated for $R$ FEEL rounds.
\begin{algorithm}[t]
\footnotesize
    \caption{Energy-Efficient FEEL}  \label{alg:synch_learning}
    \LinesNumbered
    \KwIn{available workers $K$, model size $\xi$, total bandwidth $B$}
    \KwOut{Global Model $\theta$}
    \textbf{Initialize}
    model parameters $\mathbf{\theta}_0$, learning rate $\eta$, number of epoch $\varepsilon$, threshold probability $\vartheta$, and number of FEEL rounds $R$\;
  Server collects prior information (e.g, $|\mathcal{D}_{k}|$, $ {\mathbf h}_k$, ${E}_{k}$), from available workers $K$.\;
 \For{$r=1$ to $R$}{
 Server sets the FEEL round deadline;
 Server Selects a subset of devices to take part in global model training.\;
 Server broadcast the model parameters $\mathbf{\theta}_{r-1}$ to all selected workers\;
 \For{Each selected worker $k \in K$ in parallel}{
 Worker $k$ Receives $\mathbf{\theta}_{r-1}$\;
 Worker $k$ uses \textbf{Algorithm~\ref{alg:Proposedalg2}} to filter the local samples\;
 Worker $k$ finds ${\mathbf w}_k$ as defined in ~\eqref{eq:beam}\;
 Worker $k$ Finds $T^\mathrm{ up}_k$, $P^\mathrm{ up}_k$,  $T^\mathrm{cmp}_k$, and ${f^\mathrm{cmp}_k}$ using \textbf{Algorithm \ref{alg:Golden}}\;
  Worker $k$ updates $\mathbf{\theta}_{\textbf{1}}^k$ for $E - 1$ epochs\;
  Worker $k$ sends $\mathbf{\theta}_{\textbf{r}}^k$ to the server\;
 }
  The server aggregates and fuse all models 
 }
\end{algorithm}

\section{Simulation and Numerical Results}
\label{sec:experiment}
In this section, we present the performance evaluation of proposed algorithms under FEEL settings.

\subsection{Experimental Setup}
Unless otherwise specified, we consider a FEEL environment as in Fig.~\ref{fig:sysmodel} with a bandwidth $B=10 \text{MHz}$, and background noise power is set as $\sigma^2 = 10^{-8}$. 
The distance between the edge workers and the edge server is uniformly distributed between $25 m$ and $100 m$. 
For the wireless channel model, we use Rician distribution with a Rician factor of $8$ dB, and the path loss exponent factor is $3.2$. The number of antennas is $m = 4$ for the edge server and $m = 1$ for every $k$-\textit{th} worker.
We use $P_{max} = 20$ dBm and $P_{min} = −10$ dBm, for maximum and minimum transmission power, respectively. 
A minimum and maximum CPU frequencies are set to $1$ GHz and $9$ GHz, respectively.  The simulation parameters are summarized in Tabel~\ref{tab:setuppar}. 
\begin{table}[h]
	\caption{\uppercase{ Simulation settings}}\label{parameter}
	\centering
	\begin{tabular}{|c|p{4.5cm}|}%
		\hline
		\textbf{Parameter}&\textbf{Value}\\
		\hline
		Bandwidth & 10 Mhz\\
		 Transmission Power & $P_{max} = 20$ $P_{min} = −10$ dBm \\
        Spectral Density Power & $10^{-8}$ \\
		 CPU Frequency & [1, 9] Ghz\\
        Cycles Per Sample & 20 cycle/sample\\
        Capacitance Coefficient & $2\times 10^{-28}$\\
		Model Size & 2.2 MB for MNIST, 4.7 MB for CIFAR-10\\
		Learning rate & 0.001\\
		\hline
	\end{tabular}
\label{tab:setuppar}
\end{table}

For comparison, we use the baseline FEEL Algorithm and the optimization method as in \cite{yang2020energy,mo2020energy} where all data samples are included for local training, and the updates are not synchronized. 
We use the MNIST and CIFAR-10 datasets under a realistic federated setting with i.i.d and non-i.i.d data distributions.  {We use feed-forward neural network model for MNIST and convolutional neural networks (CNN) model for CIFAR-10.  }
For the i.i.d, the dataset is randomly partitioned into $K$ pieces that correspond to the workers, and each worker is assigned one part. 
For the non-i.i.d, the data is first partitioned into $C$ parts that correspond to the classification classes.
Each part is further partitioned into different shards; then, each worker is assigned only $2$ classes. 
For each worker, the local data is split into $80\%$ for training and $20\%$ for testing. 
We utilize the mini-batch SGD as a local solver with a batch size $b=20$ and learning rate $\eta = 0.001$ and evaluate the global model every FEEL round.  {For all experiments, the results are collected and averaged over five trials.}  

\subsection{Performance Evaluation}
We evaluate the performance of the proposed approach in terms of local energy consumption, global training loss, global training accuracy, and the percentage of excluded data samples. 
Then, we show the effects of the number of workers on performance. 
\subsubsection{Impacts of The Proposed Approach on The Energy Reduction}
We conduct extensive experiments to assess the efficacy of the proposed approach in conserving the expended energy compared to benchmark approaches. 

\Cref{F:EnergyConsumption200_MNIST1,F:EnergyConsumption200_MNIST2} illustrate the cumulative and instantaneous energy consumption during the FEEL global training rounds when the learning task is performed on the MNIST dataset under non-i.i.d data distribution. 
It is observed that the proposed approach shows a significant reduction in energy consumption compared to the baselines. 
This gain is proportional to the number of FEEL global rounds. 
This is due to the fact that when more samples are injected into training at the beginning, the model needs more FEEL rounds to capture local patterns and specialize the global model. 
The model then can predict some samples with high confidence in the later FEEL rounds and then excludes the samples having less impact on the model quality for the remaining $(\varepsilon-1)$ epochs. 
This procedure reduces the time and energy needed for training, assigning more time for uploading using lower transmission power. 
Besides, it is worth noting that as the value of  $\vartheta$ decreases, more energy gains are obtained thanks to the increasing number of excluded samples. 
\begin{figure*}[!t]
\centering
	     \begin{subfigure}[b]{0.45\textwidth}
         \centering
        \includegraphics[width=\textwidth] {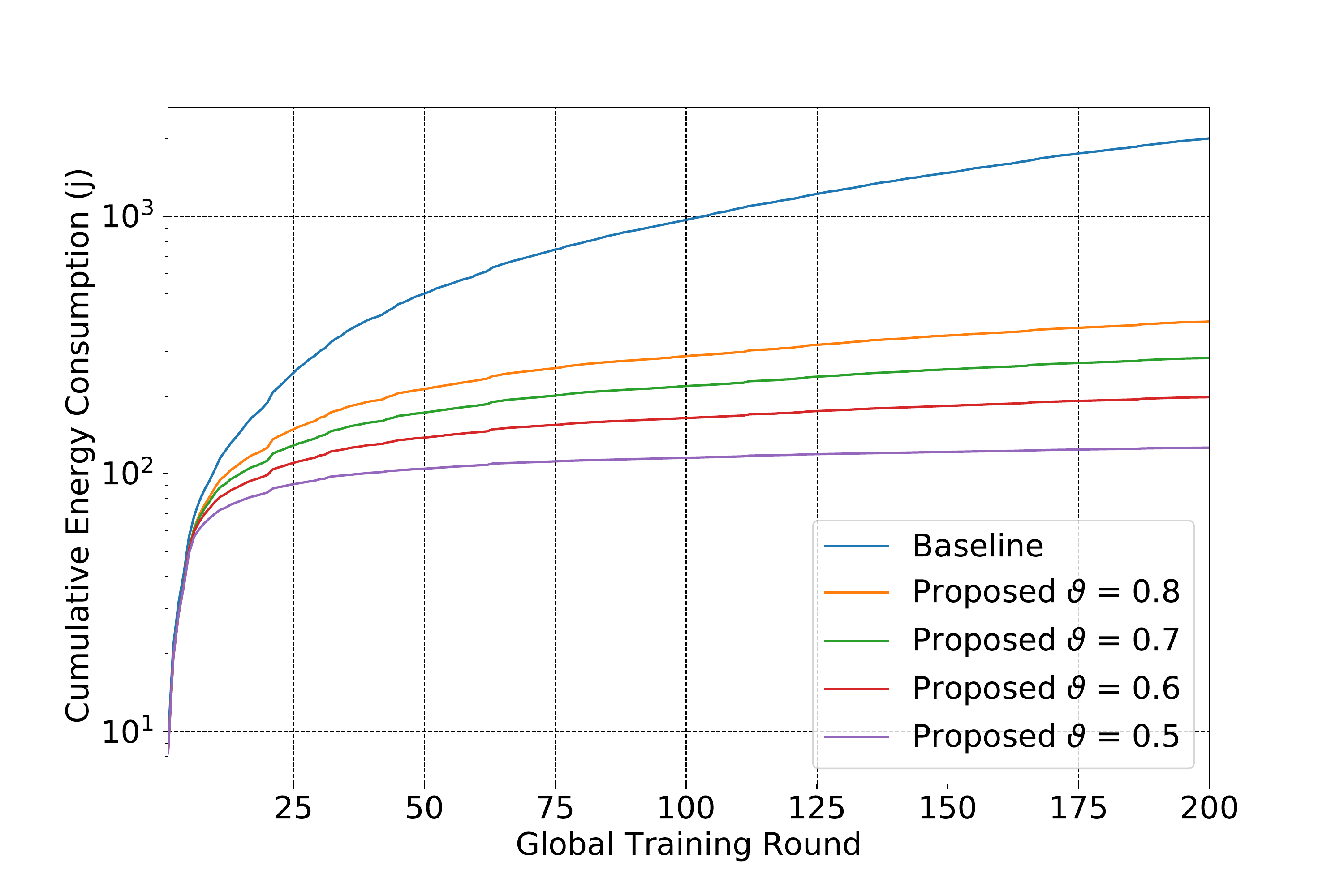}
         \caption{Cumulative Energy Consumption vs FEEL Global Round.}
         \label{F:EnergyConsumption200_MNIST1}
     \end{subfigure}
     \hfill
	     \begin{subfigure}[b]{0.45\textwidth}
         \centering
        \includegraphics[width=\textwidth] {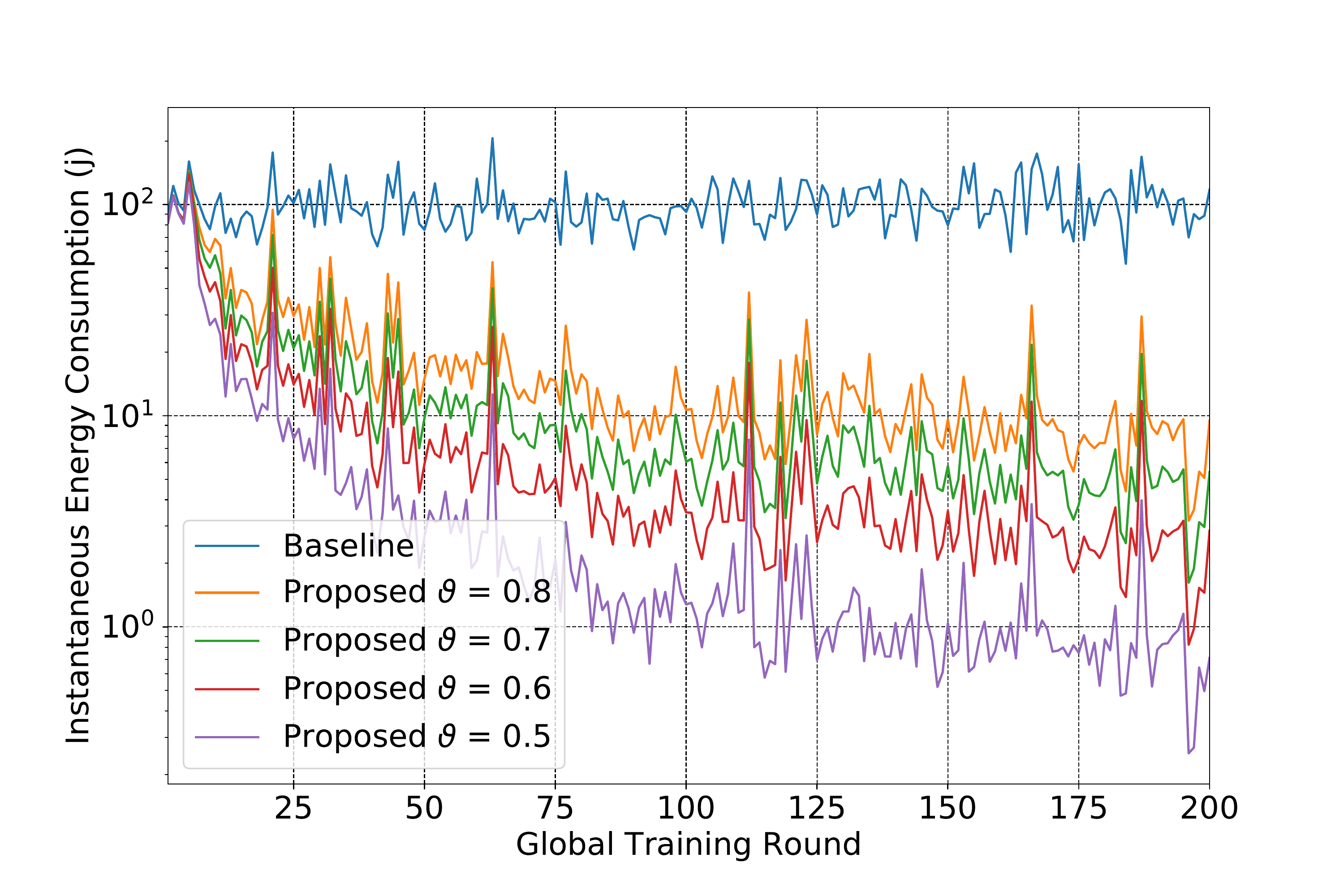}
         \caption{Instantaneous Energy Consumption vs FEEL Global Round.}
         \label{F:EnergyConsumption200_MNIST2}
     \end{subfigure}
     \hfill     
	\caption{Cumulative and Instantaneous Energy Consumption when  10\% of K  is selected and FEEL Global Rounds is $200$ (Non-i.i.d, MNIST).}
		\label{F:EnergyConsumption200_MNIST_cum_ins}
\end{figure*}

\Cref{F:EnergyConsumption200_CIFAR101,F:EnergyConsumption200_CIFAR102} show the performance of the proposed approach in terms of cumulative and instantaneous energy consumption using CIFAR-10 under non-i.i.d data distribution. 
We use this scenario to showcase the performance of the proposed approach in conjunction with more complex learning tasks. 
We can note that, in general, training the model using CIFAR-10 consumes more energy compared to the MNIST. 
Nevertheless, the proposed approach still provides significant performance gain, and it conserves substantial energy even through the dataset and corresponding model are complex. 
This stems from the ability of the global model to exclude less important data samples leading to a reduction in the computation time and energy while providing more flexibility to optimize the transmission energy as more time is assigned for transmission.
 \begin{figure*}[!t]
\centering
	     \begin{subfigure}[b]{0.45\textwidth}
         \centering
        \includegraphics[width=\textwidth] {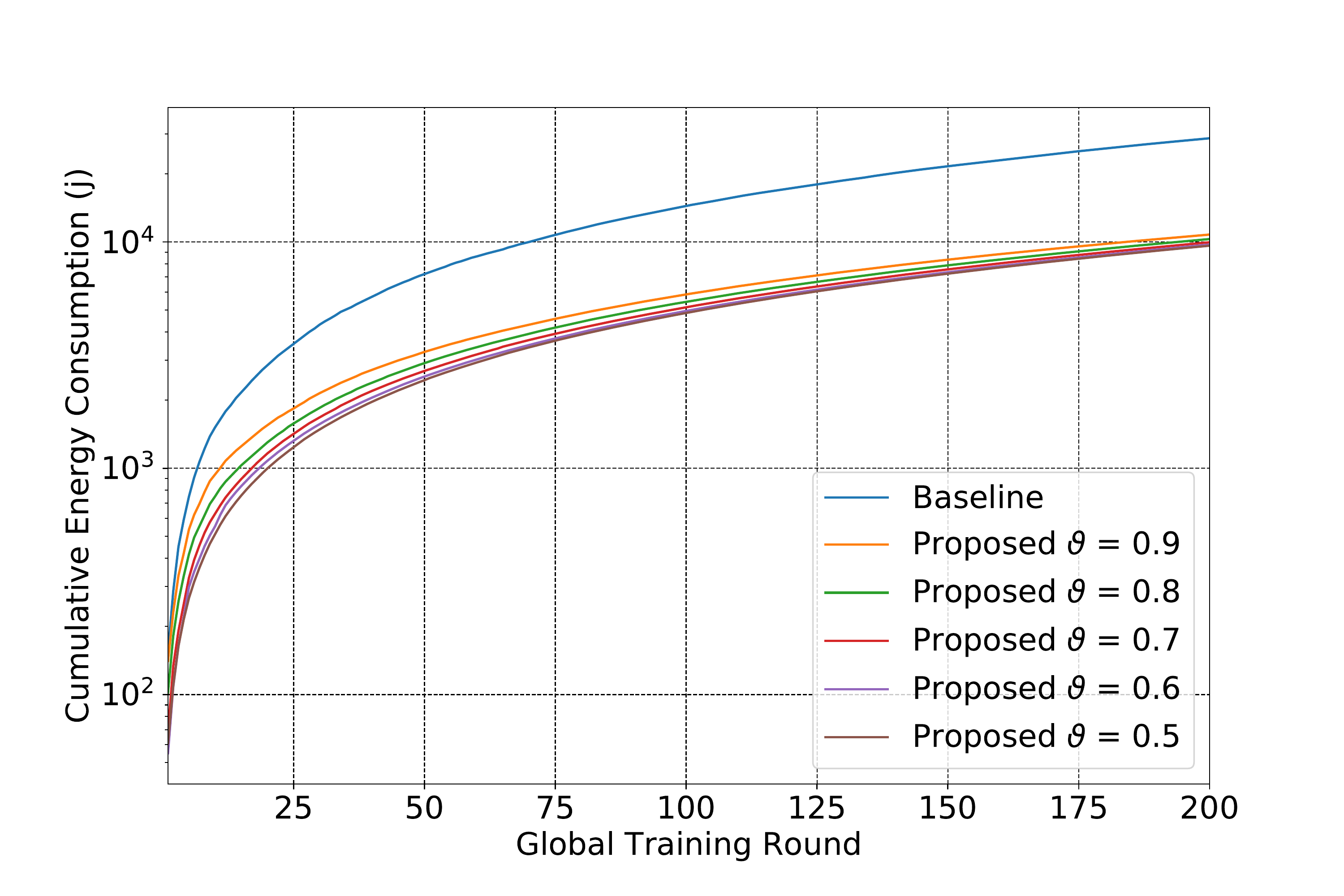}
         \caption{Cumulative Energy Consumption vs FEEL Global Round.}
         \label{F:EnergyConsumption200_CIFAR101}
     \end{subfigure}
     \hfill
	     \begin{subfigure}[b]{0.45\textwidth}
         \centering
        \includegraphics[width=\textwidth] {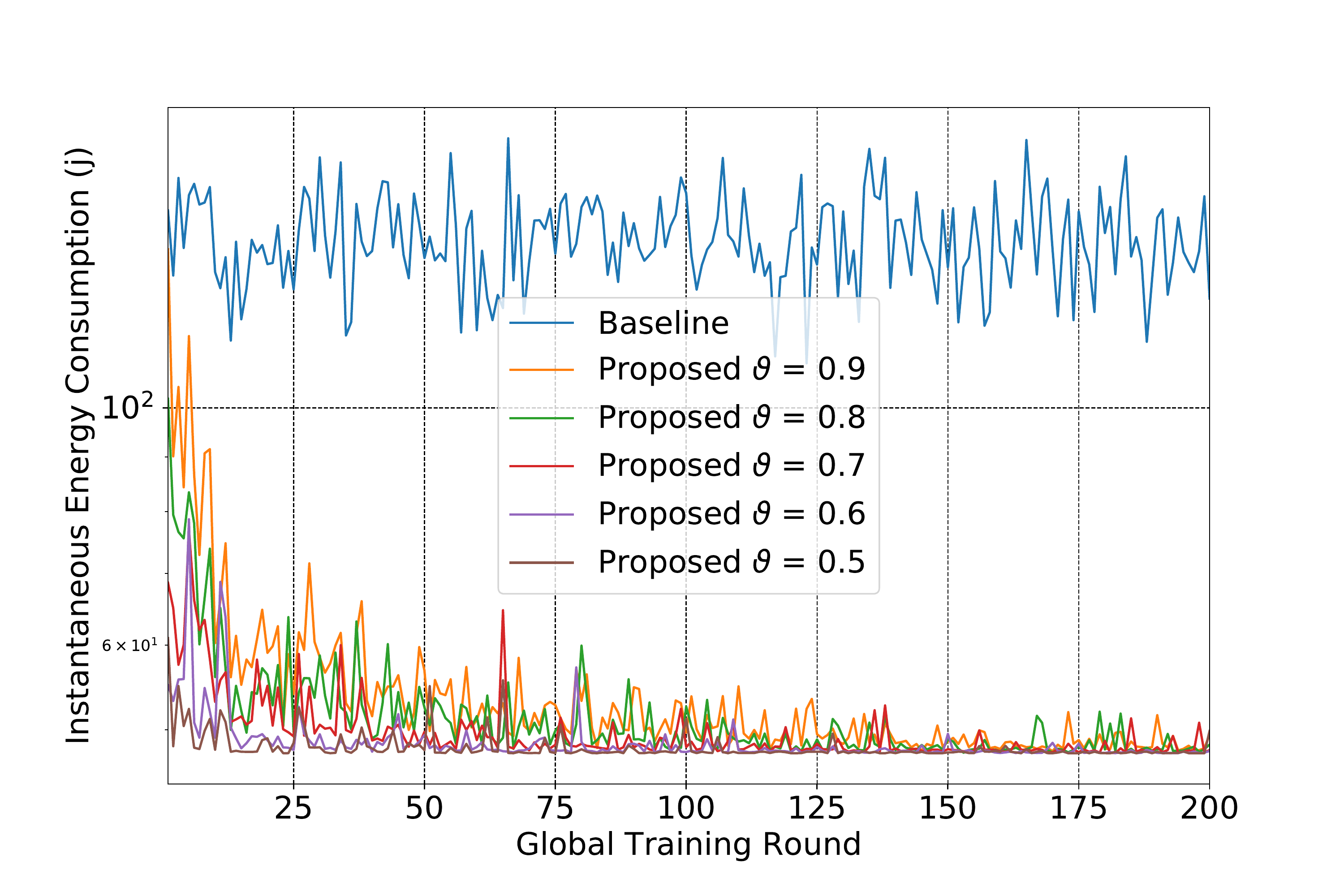}
         \caption{Instantaneous Energy Consumption vs FEEL Global Round.}
         \label{F:EnergyConsumption200_CIFAR102}
     \end{subfigure}
     \hfill     
	\caption{Cumulative and Instantaneous Energy Consumption when  10\% of K  is selected and FEEL Global Rounds is $200$ (Non-i.i.d, CIFAR-10).}
		\label{F:EnergyConsumption200_CIFAR_cum_ins}
\end{figure*}

Further, \Cref{F:EnergyConsumption200_CIFAR101_iid,F:EnergyConsumption200_CIFAR102_iid} display the performance of the proposed approach under i.i.d data distribution. 
From these figures, it is concluded that the proposed approach still achieves better performance gains compared to the baselines if the data is i.i.d. The resulting gain stems from capturing the data patterns during FEEL rounds, indicating that it is effective to use data filtering even when the data is i.i.id.

\begin{figure*}[!t]
\centering
	     \begin{subfigure}[b]{0.45\textwidth}
         \centering
        \includegraphics[width=\textwidth] {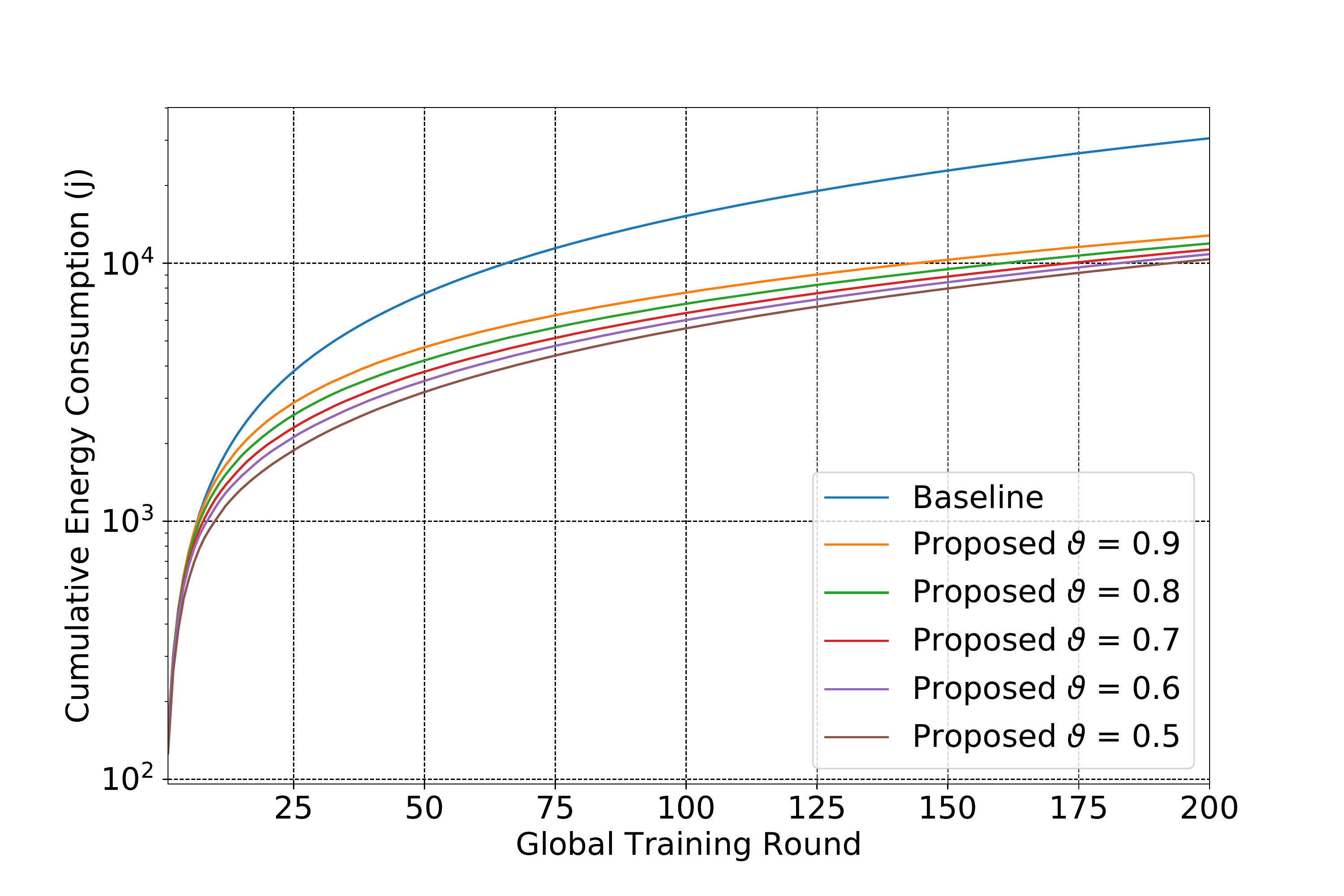}
         \caption{Cumulative Energy Consumption vs FEEL GlobaL Round.}
         \label{F:EnergyConsumption200_CIFAR101_iid}
     \end{subfigure}
     \hfill
	     \begin{subfigure}[b]{0.45\textwidth}
         \centering
        \includegraphics[width=\textwidth] {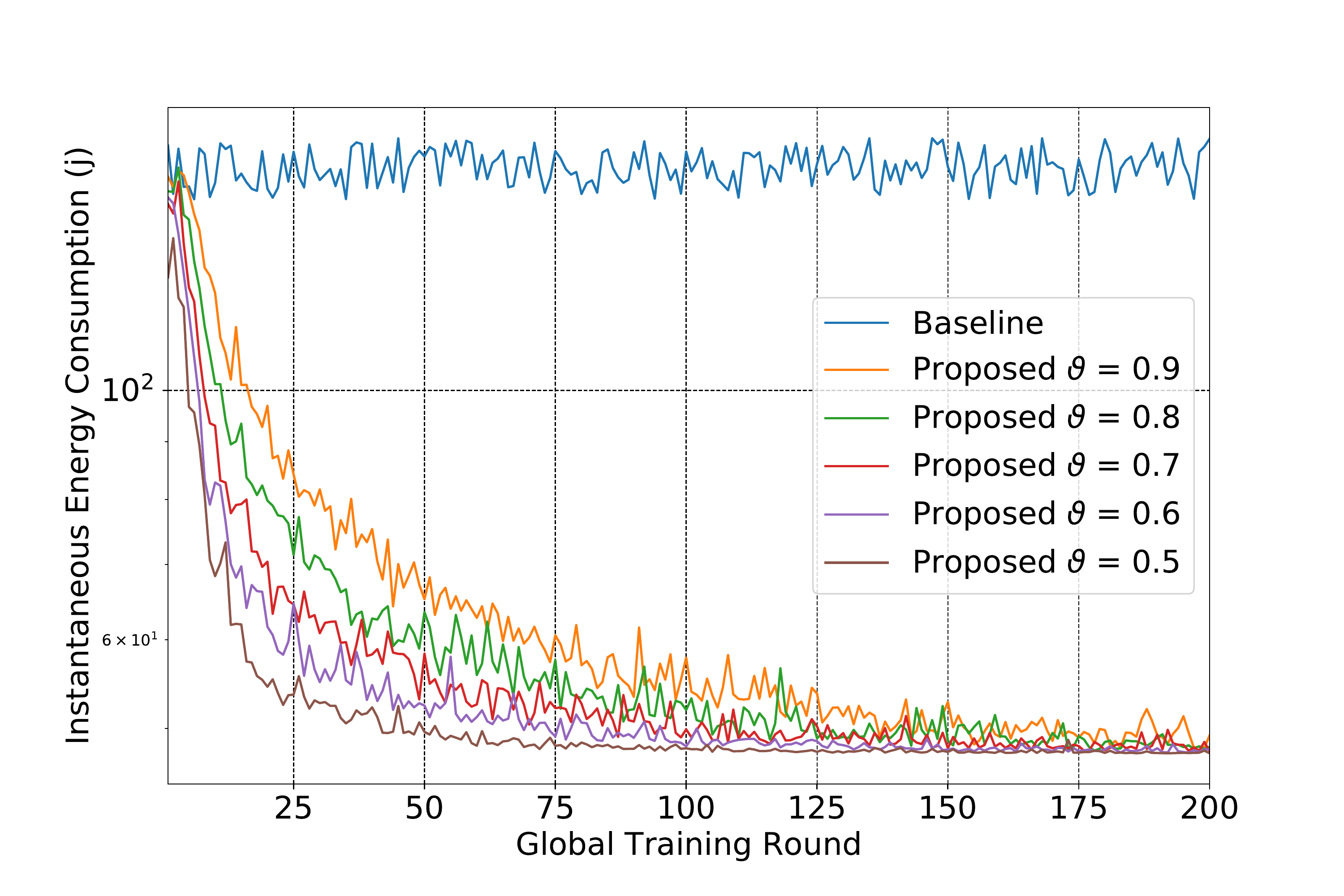}
         \caption{Instantaneous Energy Consumption vs FEEL Global Round.}
         \label{F:EnergyConsumption200_CIFAR102_iid}
     \end{subfigure}
     \hfill     
	\caption{Cumulative and Instantaneous Energy Consumption when  10\% of K  is selected and FEEL Global Rounds is $200$ (i.i.d, CIFAR-10).}
		\label{F:EnergyConsumption200_CIFAR_cum_ins_iid}
\end{figure*}

\subsubsection{Impacts of The Proposed Approach on The Testing Loss and Accuracy} 

\Cref{F:MNIST_loss200,F:MNIST_accu200} show the identification accuracy and loss of handwritten digits (MNIST) when the number of FEEL global rounds is $200$ and $\vartheta =  0.5, 0.6, 0.7,$ and $0.8$. 
From these figures, the testing accuracy is not negatively affected when using our proposed training approach. It is worth noting that excluding samples that can be predicted with high probability does not affect the performance as  similar accuracy and loss are still achieved, especially in scenarios that have $\vartheta > 0.70$. 

\begin{figure*}[!t]
\centering
	     \begin{subfigure}[b]{0.45\textwidth}
         \centering
        \includegraphics[width=\textwidth] {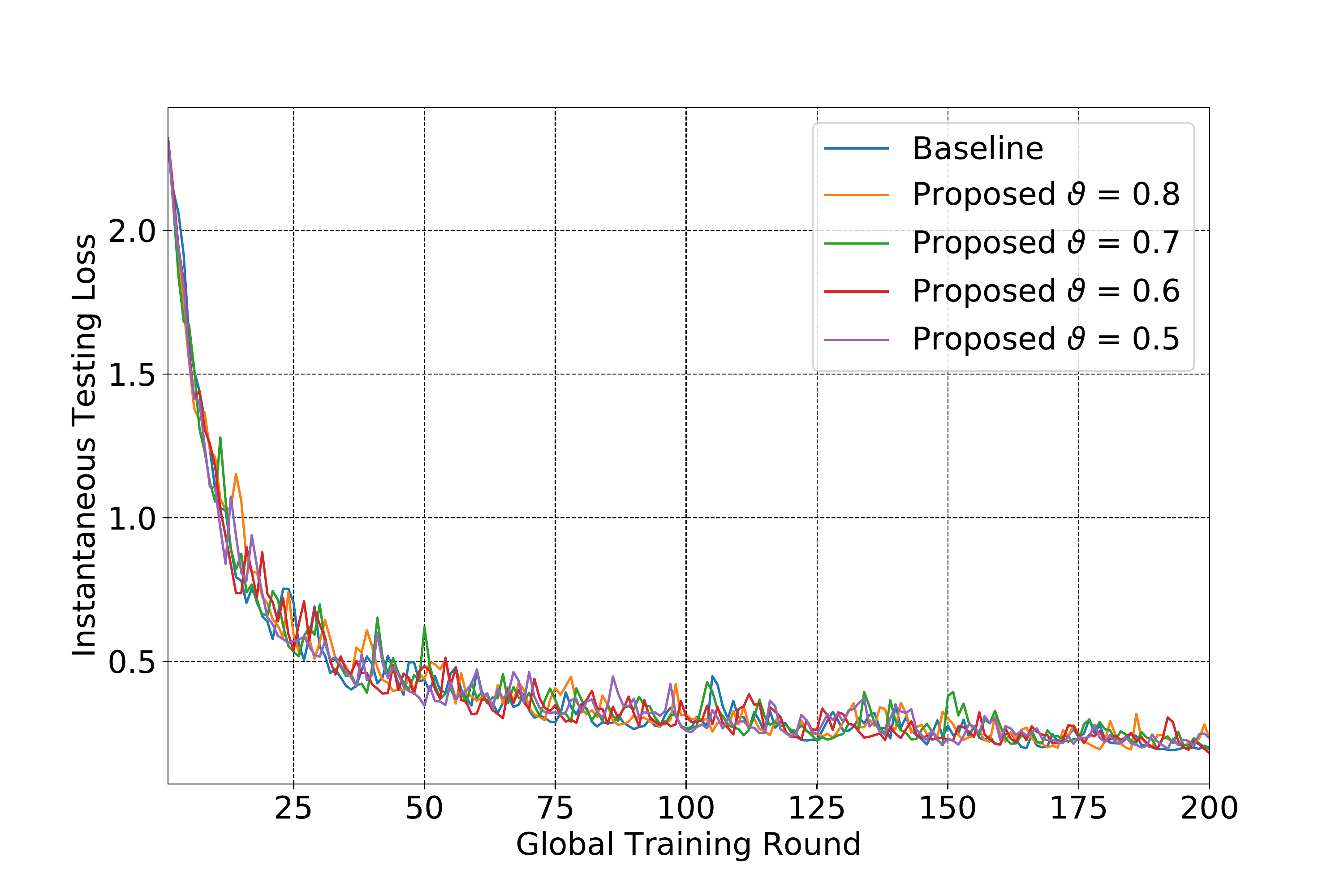}
         \caption{Testing Loss vs FEEL Global Round.}
         \label{F:MNIST_loss200}
     \end{subfigure}
     \hfill
	     \begin{subfigure}[b]{0.45\textwidth}
         \centering
        \includegraphics[width=\textwidth] {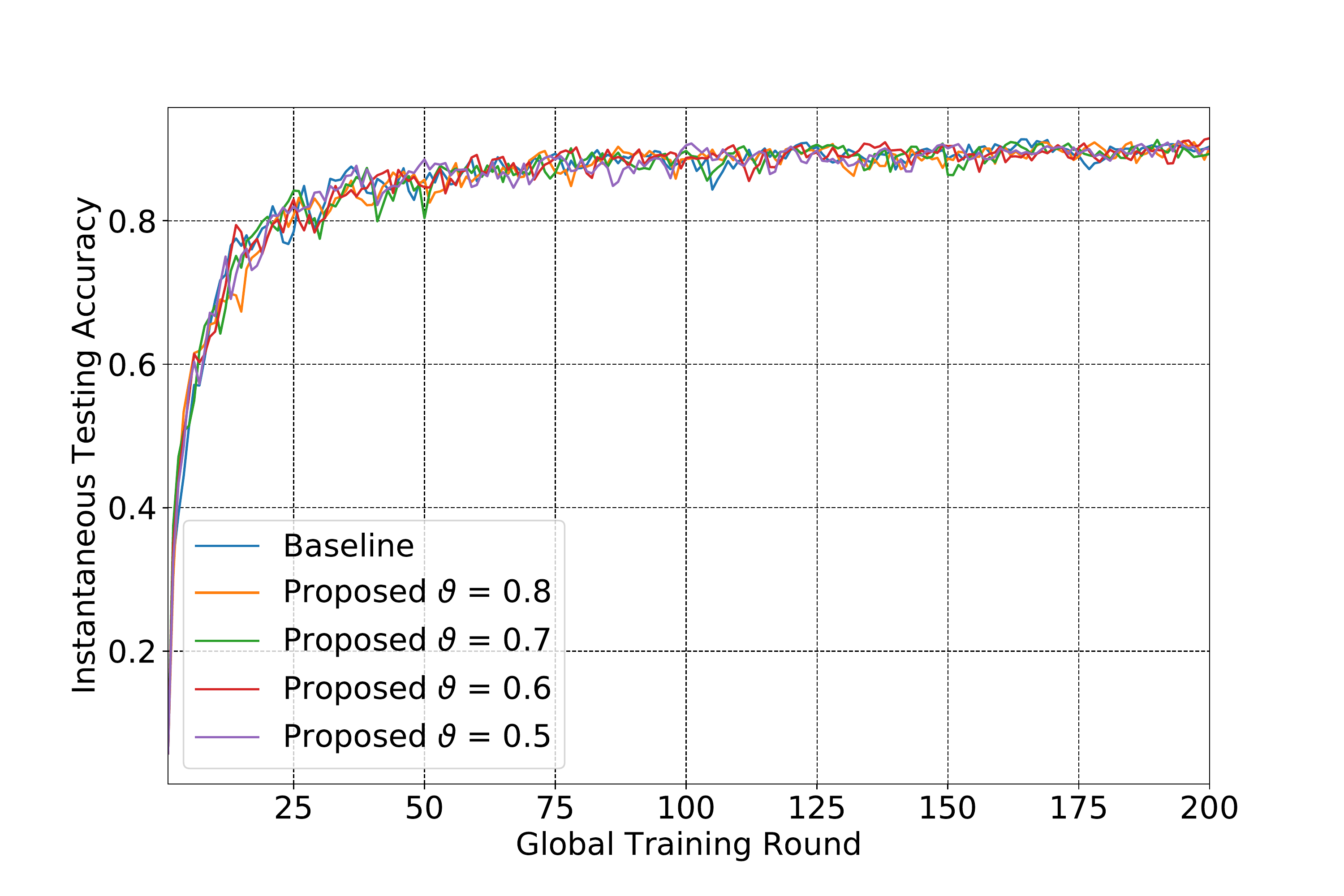}
         \caption{Testing Accuracy vs FEEL Global Round.}
         \label{F:MNIST_accu200}
     \end{subfigure}
     \hfill     
	\caption{Testing Loss and Accuracy when  10\% of K  is selected and FEEL Global Rounds is $200$ (Non-i.i.d, MNIST).}
		\label{F:MNIST_test_loss_noniid}
\end{figure*}

Furthermore, \Cref{F:cifar_loss200,F:cifar_accu200} show the identification accuracy and loss of photo classification   (CIFAR-10) when the number of FEEL global rounds is $200$ and $\vartheta =  0.5, 0.6, 0.7, 0.8$ and $0.9$. 
From both figures, it is evident that the proposed approach provides similar accuracy and loss, especially when the threshold probability is higher than $0.80$. 
However, in contrast to MNIST, both accuracy and loss worsen when the threshold probability is lower than $0.70$, as we can see when $\vartheta = 0.50$. 
This is due to the fact that most of the excluded samples are harder to distinguish because of the limited number of samples used to train the local models for the remaining epochs. 
This indicates that for more complex learning tasks such as CIFAR-10, it is better to choose the threshold probability $\vartheta$ to be closer to $1$. 

\begin{figure*}[!t]
\centering
	     \begin{subfigure}[b]{0.45\textwidth}
         \centering
        \includegraphics[width=\textwidth] {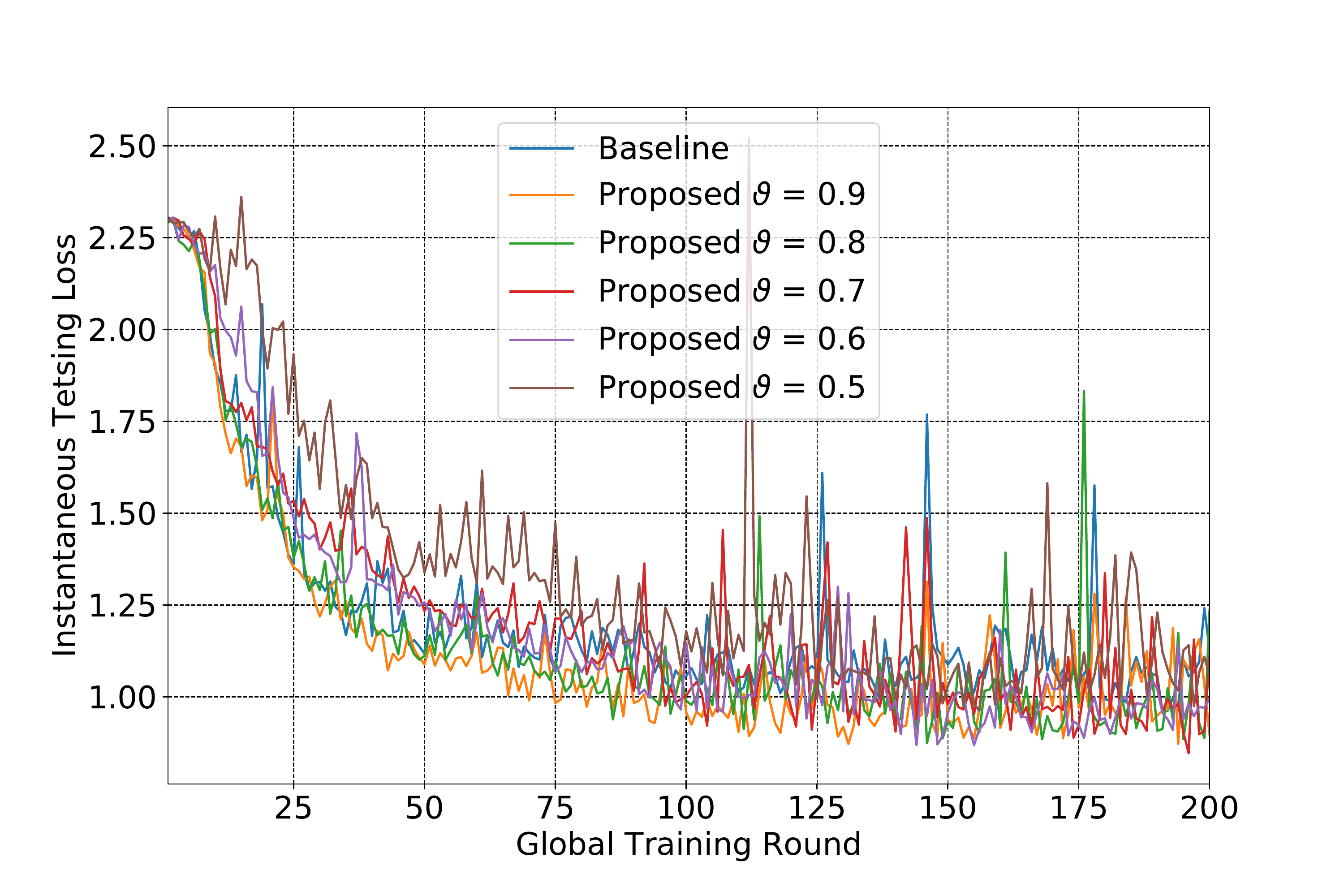}
         \caption{Testing Loss vs FEEL Global Round.}
         \label{F:cifar_loss200}
     \end{subfigure}
     \hfill
	     \begin{subfigure}[b]{0.45\textwidth}
         \centering
        \includegraphics[width=\textwidth] {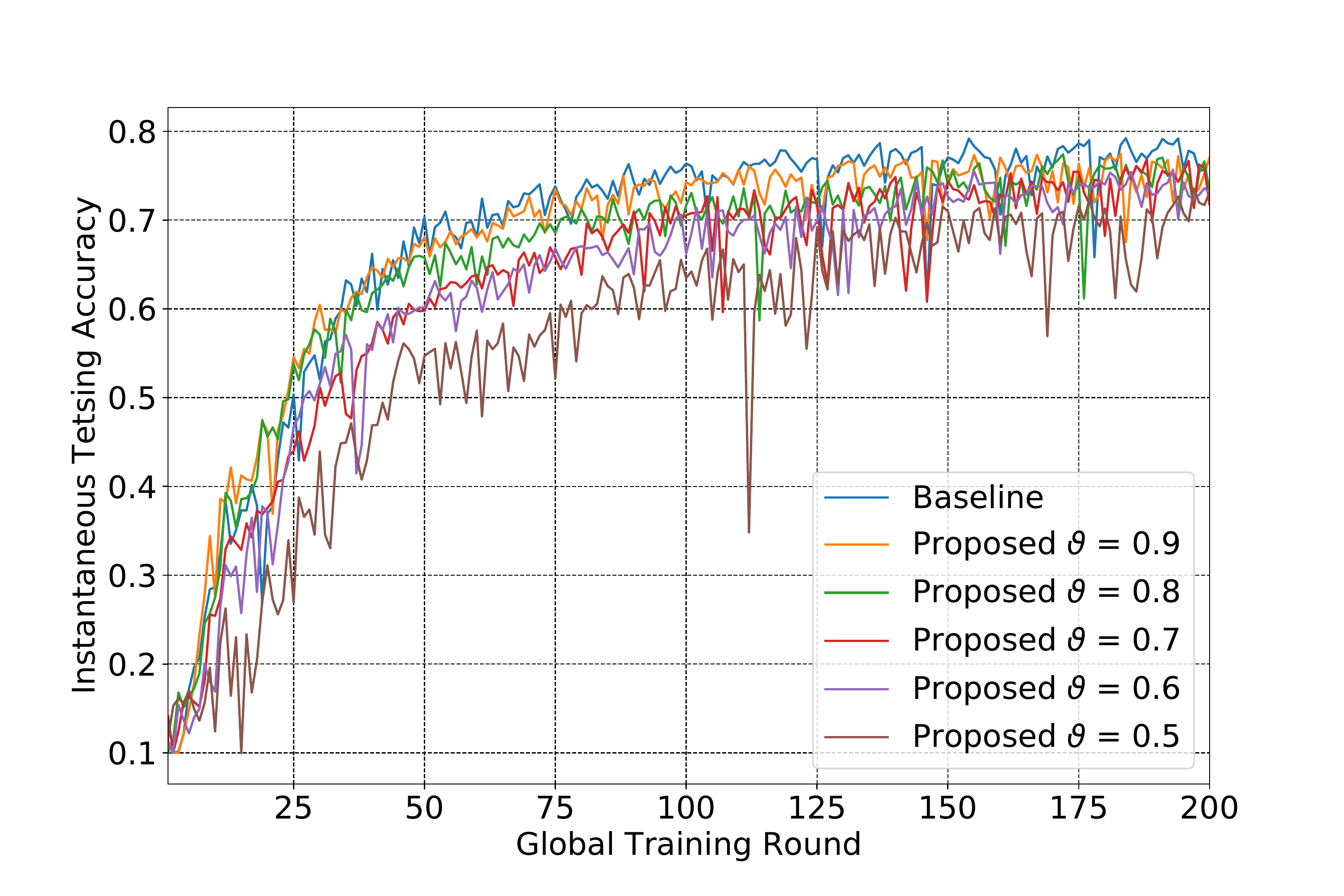}
         \caption{Testing Accuracy vs FEEL Global Round.}
         \label{F:cifar_accu200}
     \end{subfigure}
     \hfill     
	\caption{Testing Loss and Accuracy when  10\% of K  is selected and FEEL Global Rounds is $200$ (Non-i.i.d, CIFAR-10).}
		\label{F:cifar10_test_loss_noniid}
\end{figure*}

For the i.i.d. scenario, as shown in \Cref{F:cifar_loss200_iid,F:cifar_accu200_iid}, the proposed approach provides a faster convergence rate than non-i.i.d data distribution when 10\% of the clients participate in every FEEL round. 
This is due to the fact that the loss function
in i.i.d data distribution is more smooth and more convex than the non-i.i.d. datasets. The achievable accuracy can reach  90\%, which is about 10\% higher than the non-i.i.d. 
In contrast to non-i.i.d, we can observe that all thresholds greater than $50\%$  almost attain similar accuracy and loss while conserving much more energy than the baseline.  
\begin{figure*}[!t]
\centering
	     \begin{subfigure}[b]{0.45\textwidth}
         \centering
        \includegraphics[width=\textwidth] {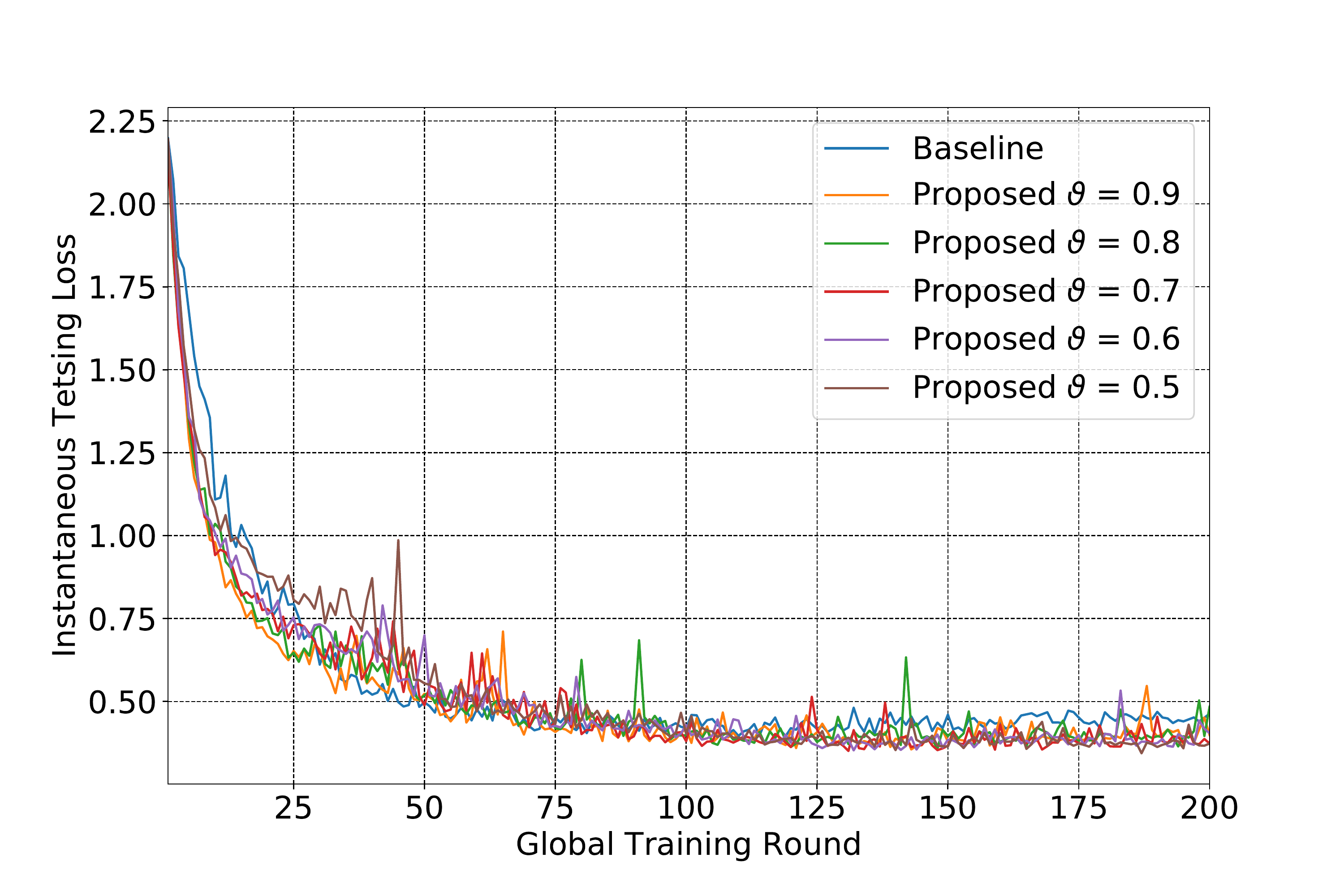}
         \caption{Testing Loss vs FEEL Global Round.}
         \label{F:cifar_loss200_iid}
     \end{subfigure}
     \hfill
	     \begin{subfigure}[b]{0.45\textwidth}
         \centering
        \includegraphics[width=\textwidth] {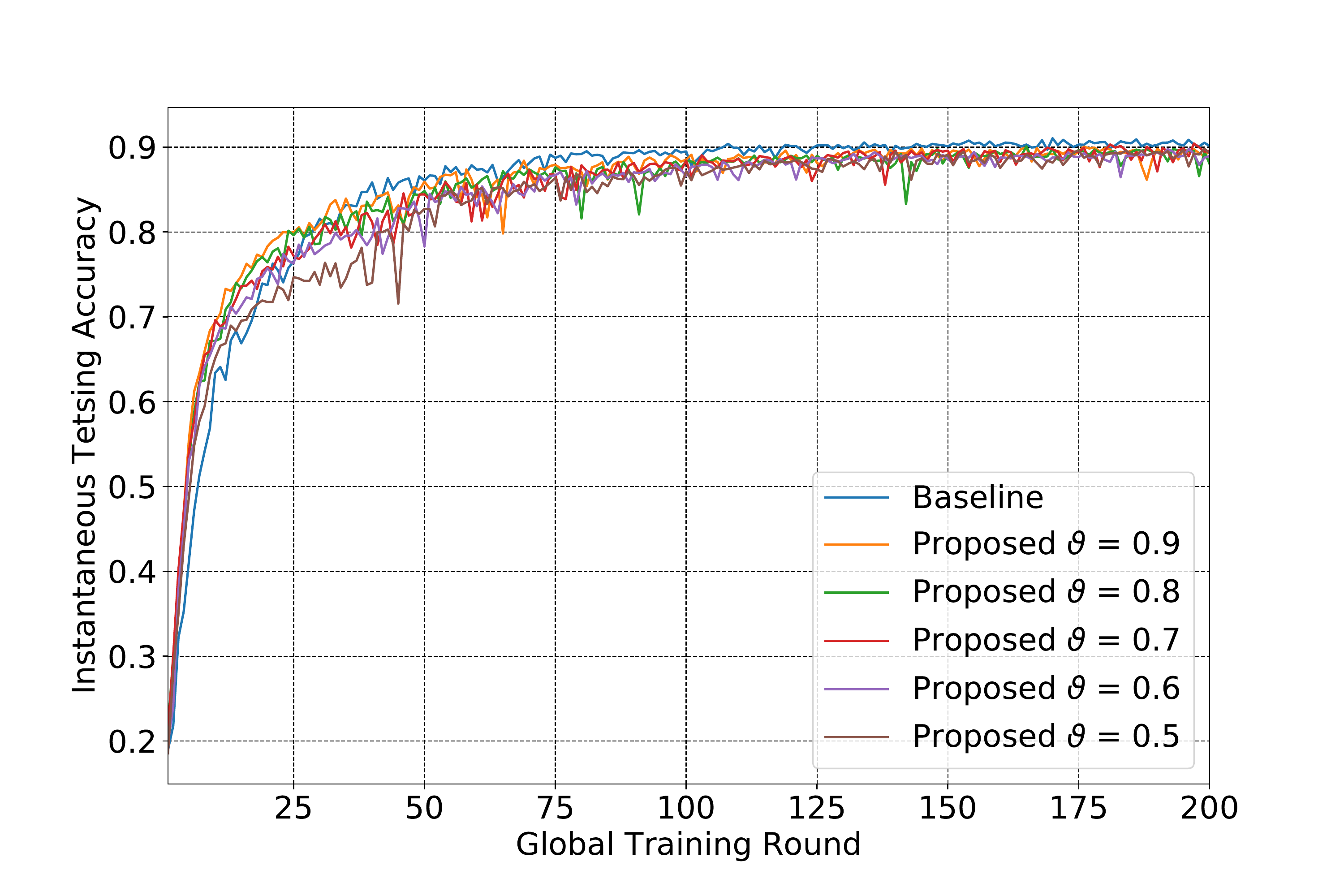}
         \caption{Testing Accuracy vs FEEL Global Round.}
         \label{F:cifar_accu200_iid}
     \end{subfigure}
     \hfill     
	\caption{Testing Loss and Accuracy when  10\% of K  is selected and FEEL Global Rounds is $200$ (i.i.d, CIFAR-10).}
		\label{F:cifar10_test_loss_iid}
\end{figure*}

Overall, our proposed approach provides significant energy efficiency improvements, therefore, encouraging real-life deployments of synchronized edge intelligence while maintaining privacy. 
The performance gain stems from the excluded data samples. These gains can be as high as 90\% of the total samples, as can be observed from our experiments. 
In detail, \Cref{F:samples200_Cifar1,F:samples200_Cifar2} show that at the beginning of the training process, in the initial FEEL rounds,  about $40$\% of the data samples are excluded when  $\vartheta = 0.90$. 
Nevertheless, this percentage increases over time, and it can reach $90$\% of the original data samples as exhibited in \Cref{F:samples200_Cifar2}.  
\begin{figure*}[!t]
\centering
	     \begin{subfigure}[b]{0.45\textwidth}
         \centering
        \includegraphics[width=\textwidth] {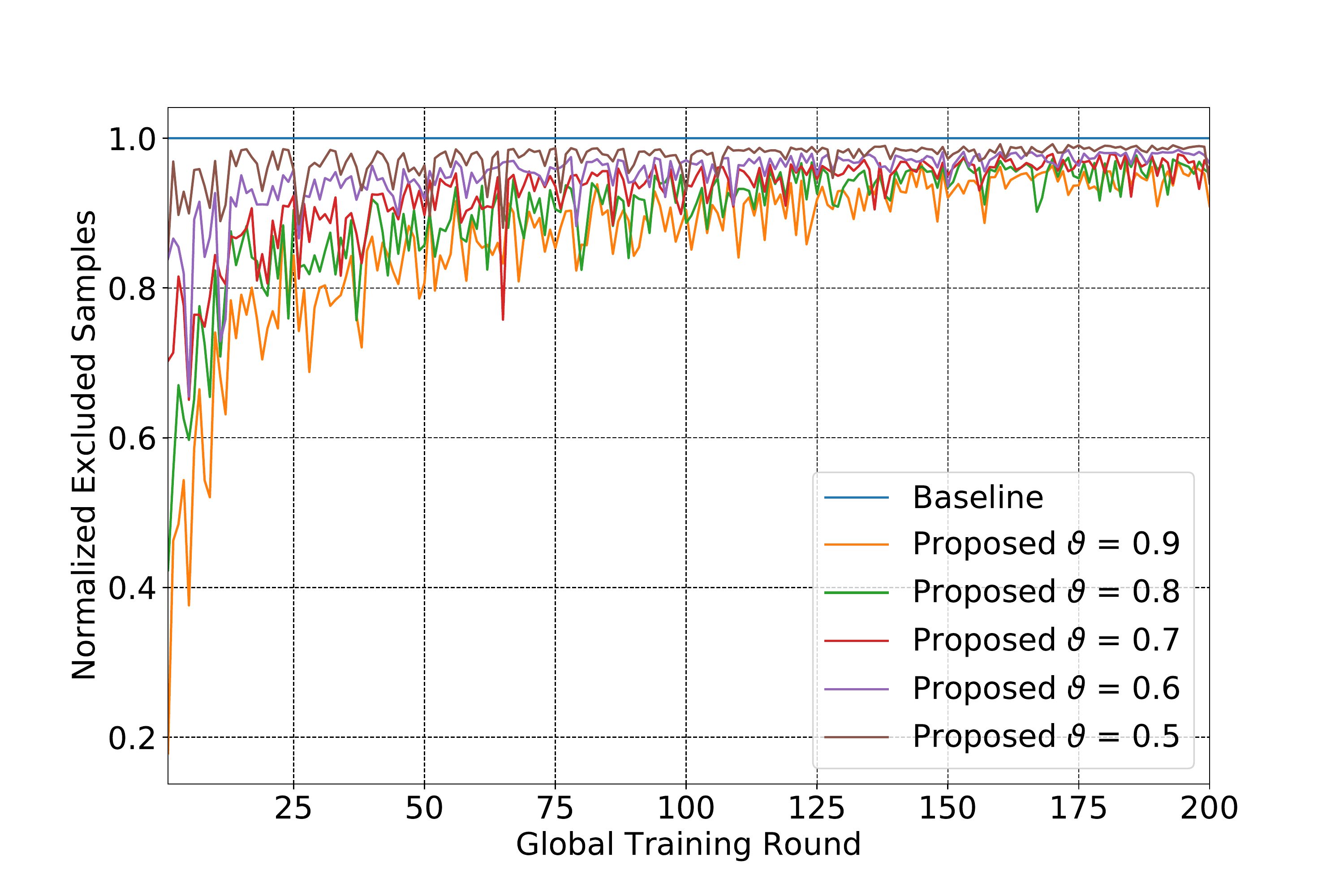}
         \caption{Normalized excluded samples to the baseline vs FEEL Global Round.}
         \label{F:samples200_Cifar1}
     \end{subfigure}
     \hfill
	     \begin{subfigure}[b]{0.45\textwidth}
         \centering
        \includegraphics[width=\textwidth] {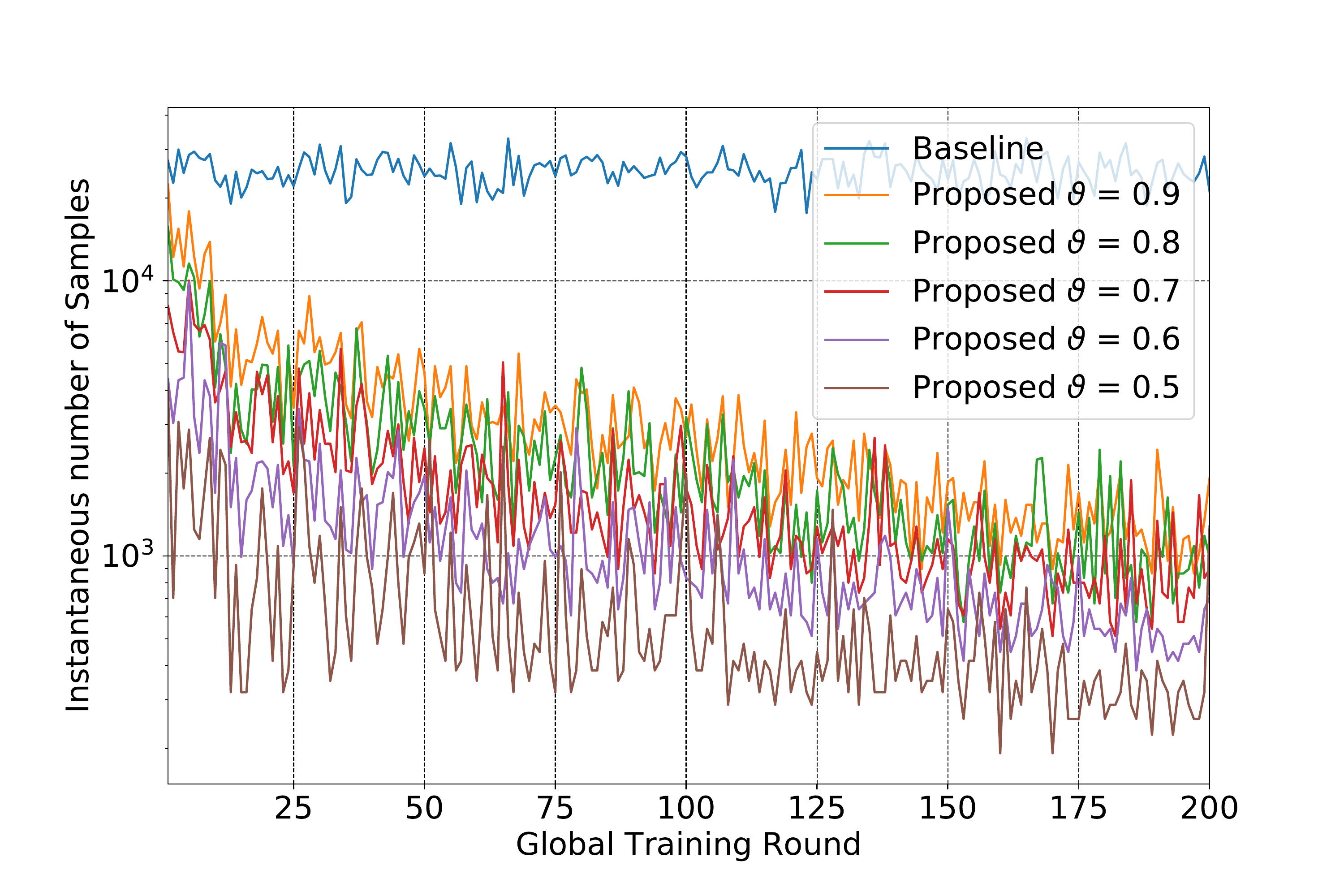}
         \caption{Instantaneous included samples vs FEEL Global Round.}
         \label{F:samples200_Cifar2}
     \end{subfigure}
     \hfill     
	\caption{Normalized excluded samples and Instantaneous included samples when  10\% of K  is selected and FEEL Global Rounds is $200$ (Non-i.i.d, CIFAR-10).}
		\label{F:samples200_Cifar}
\end{figure*}

\subsubsection{Impacts of The Number of Selected Workers on The Performance}
We explore the effects of the number of selected workers on energy consumption and the convergence rate for further analysis. 
We use CIFAR-10 for learning tasks, assuming that only 5\% of users participate in every FEEL training round. \Cref{F:EnergyConsumption200_CIFAR101_u_5,F:EnergyConsumption200_CIFAR102_u_5} illustrate the cumulative and instantaneous energy consumption vis-a-vis FEEL global rounds. It can be noticed that our proposed approach conserves energy regardless of the number of workers, while it is clear that as the number of workers decreases, more energy is consumed per worker. 
This stems from the nature of non-i.i.d data distribution. 
As fewer workers are involved during the training process, the model cannot learn and identify more heterogeneous and diverse data, leading to a slower convergence rate as shown on \Cref{F:cifar_loss200_u_5,F:cifar_accu200_u_5}. 
These figures show the instantaneous results of the testing loss and accuracy when 5\% of the workers participate in the training process. 
It is observed that the higher the threshold probability, the best performance gain compared to the benchmarks regardless of the number of participating workers. 
\begin{figure*}[!t]
\centering
	     \begin{subfigure}[b]{0.45\textwidth}
         \centering
        \includegraphics[width=\textwidth] {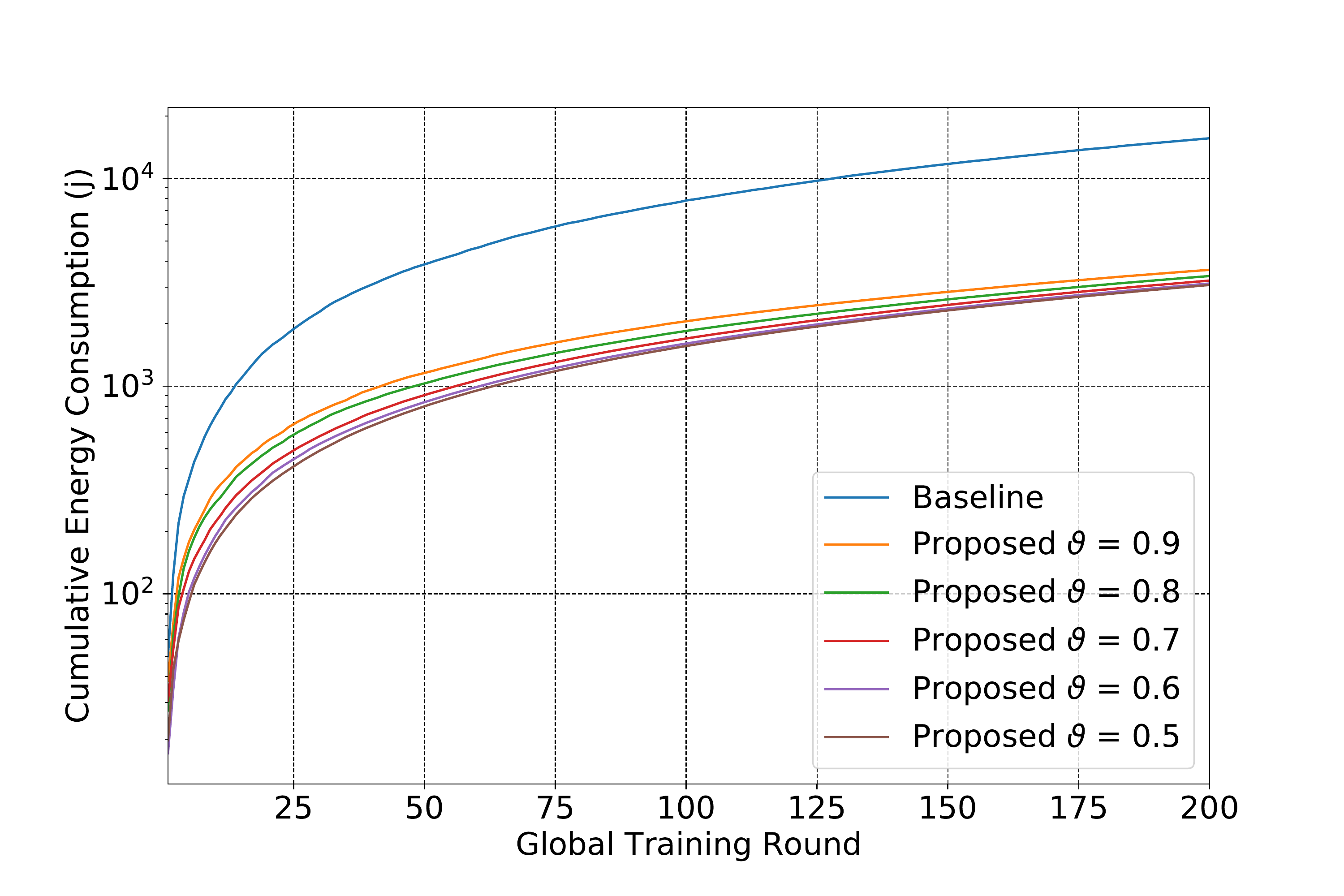}
         \caption{Cumulative Energy Consumption vs FEEL Global Round.}
         \label{F:EnergyConsumption200_CIFAR101_u_5}
     \end{subfigure}
     \hfill
	     \begin{subfigure}[b]{0.45\textwidth}
         \centering
        \includegraphics[width=\textwidth] {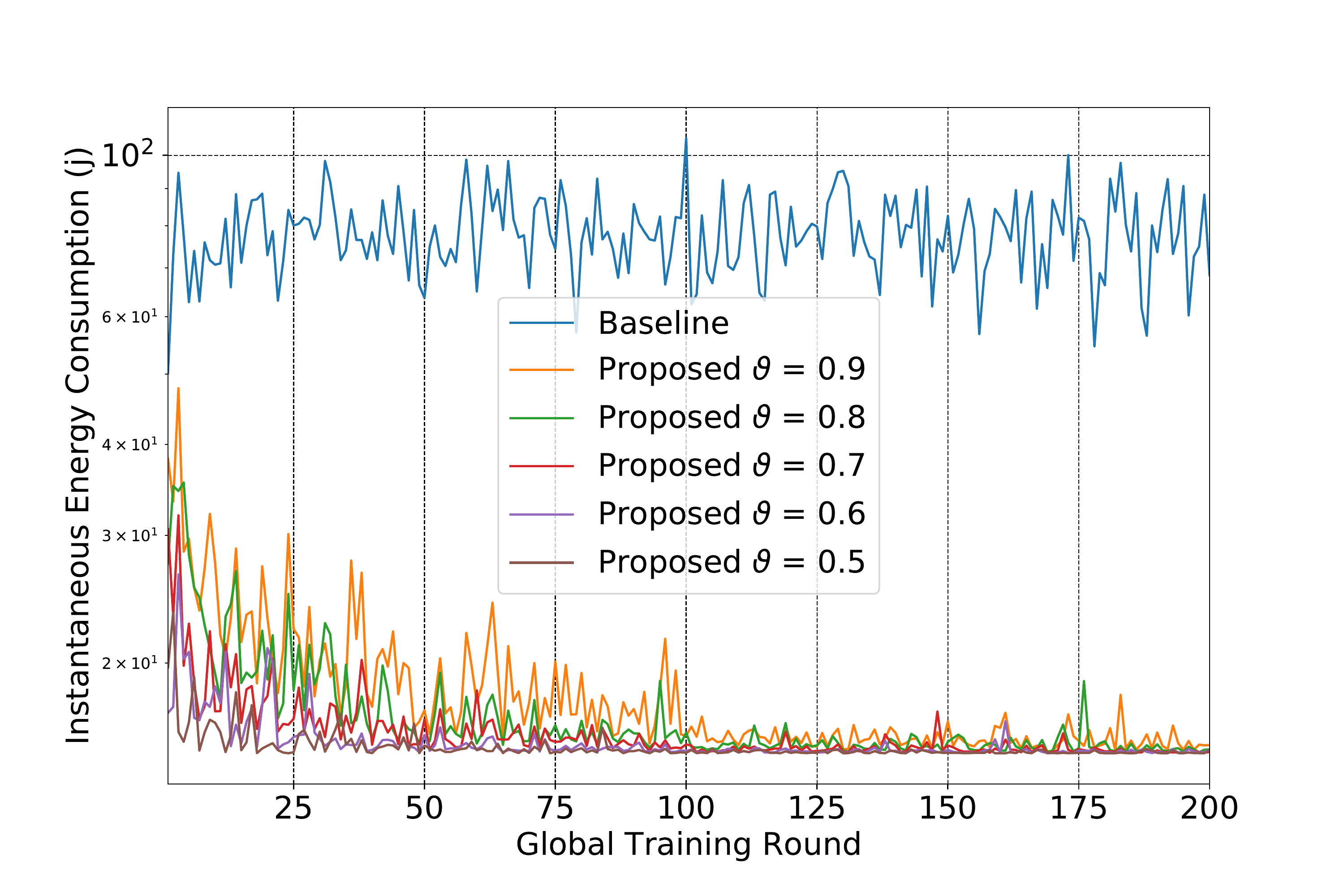}
         \caption{Instantaneous Energy Consumption vs FEEL Global Round.}
         \label{F:EnergyConsumption200_CIFAR102_u_5}
     \end{subfigure}
     \hfill     
	\caption{Cumulative and Instantaneous Energy Consumption when  5\% of K  is selected and FEEL Global Rounds is $200$ (Non-i.i.d, CIFAR-10).}
		\label{F:EnergyConsumption200_CIFAR_cum_ins_u_5}
\end{figure*}

\begin{figure*}[htbp]
\centering
	     \begin{subfigure}[b]{0.45\textwidth}
         \includegraphics[width=\textwidth] {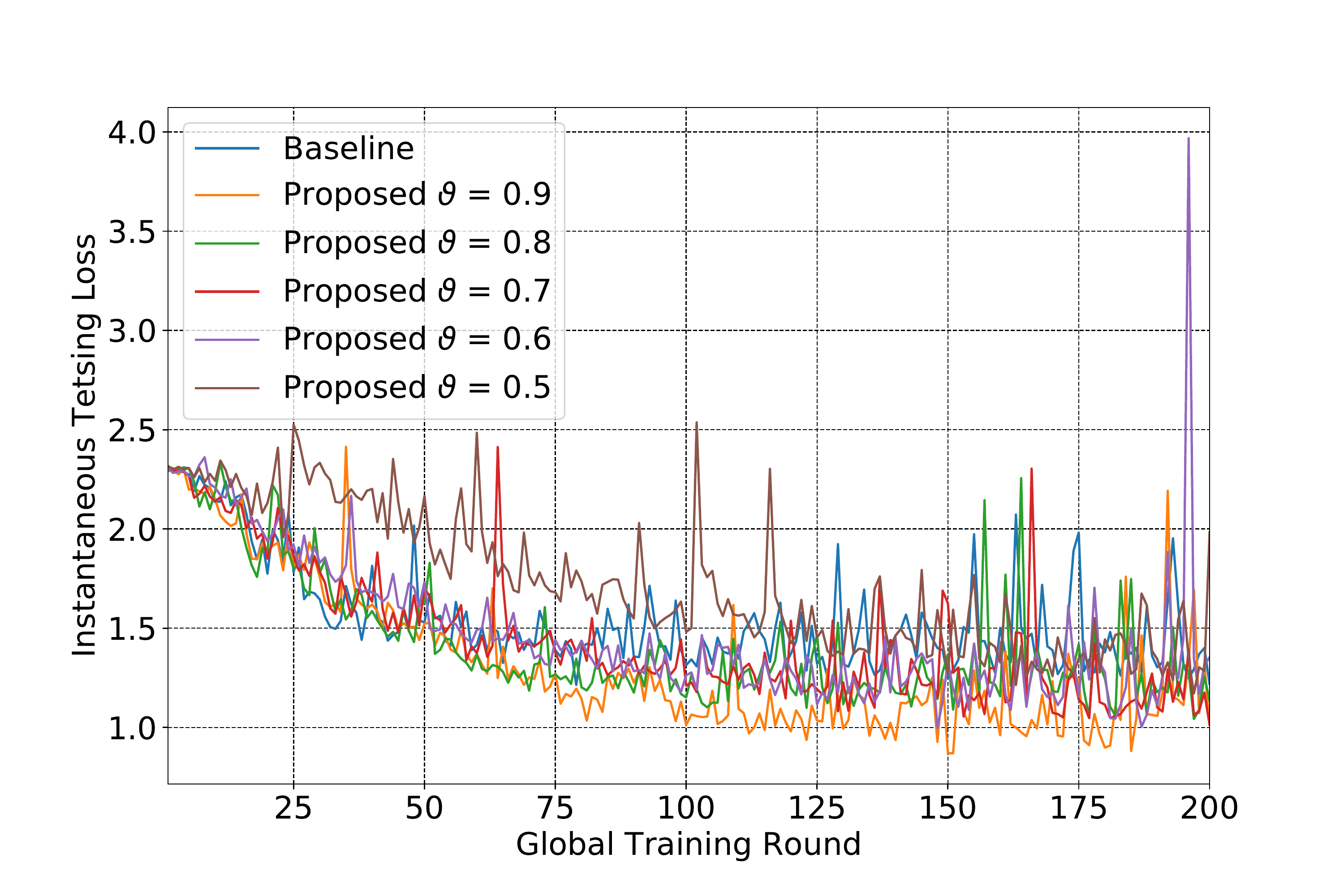}
         \caption{Testing Loss vs FEEL Global Round.}
         \label{F:cifar_loss200_u_5}
     \end{subfigure}
     \hfill
	     \begin{subfigure}[b]{0.45\textwidth}
         \includegraphics[width=\textwidth]{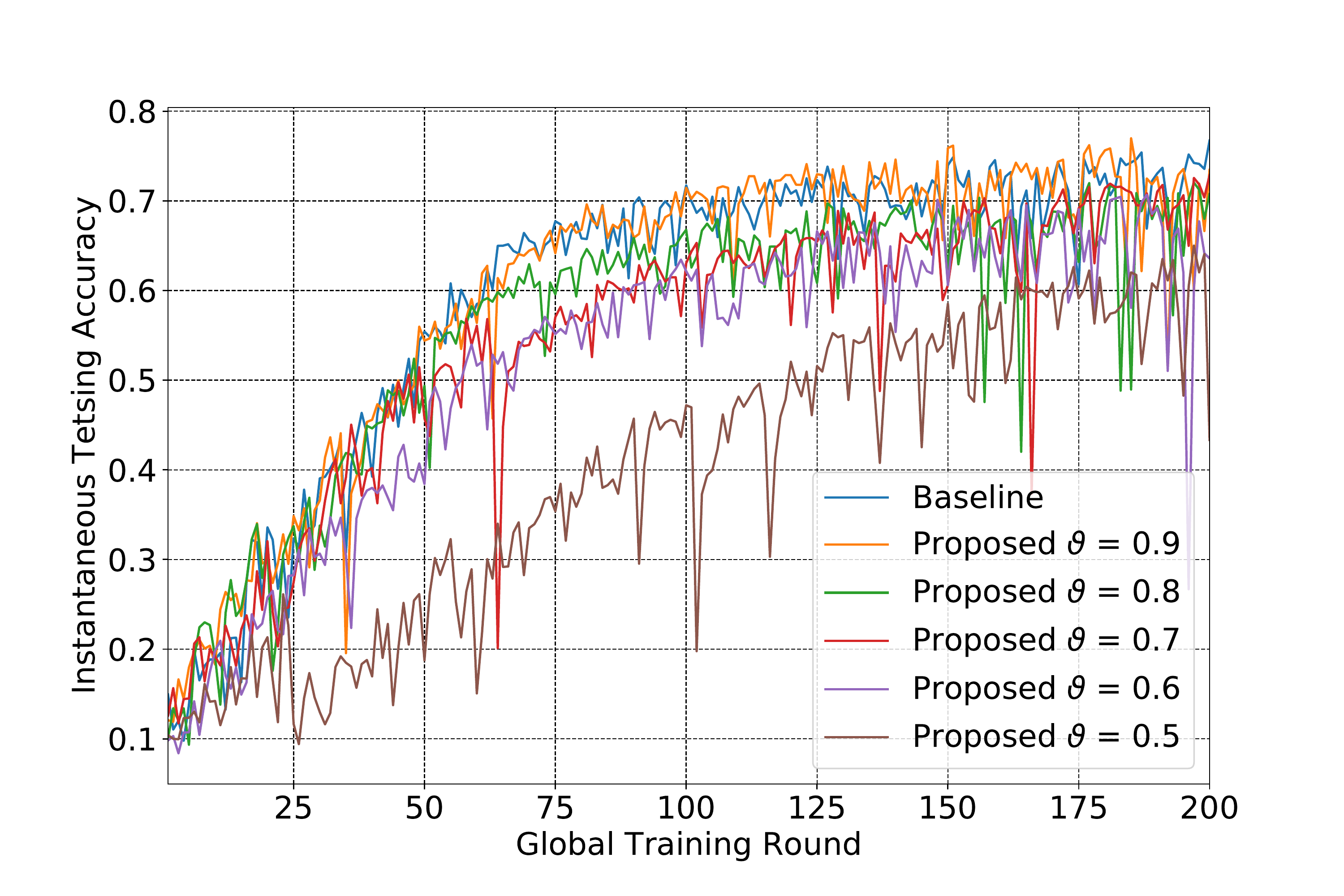}
         \caption{Testing Accuracy vs FEEL Global Round.}
         \label{F:cifar_accu200_u_5}
     \end{subfigure}
     \hfill     
	\caption{Testing Loss and Accuracy when  10\% of K  is selected and FEEL Global Rounds is $200$ (Non-i.i.d, CIFAR-10).}
		\label{F:cifar10_test_loss_noniid_u_5}
\end{figure*}

\subsection{Lesson Learned}
The main lessons and conclusions that can be drawn from our experiments are summarized as follows:
\begin{itemize}
    \item Excluding data samples that do not affect the learning performance can help with energy conservation in FEEL settings regardless of the nature of the data distribution (i.e., i.i.d or not). This is clearly illustrated in \Cref{F:EnergyConsumption200_MNIST_cum_ins,F:EnergyConsumption200_CIFAR_cum_ins}.  
    \item Excluding the data based on higher threshold probability tends to provide more performance gains compared to scenarios that utilize a lower threshold.
    \item Data exclusion is independent of the learning tasks (simple or complex). 
    This has been illustrated  deminstrating that the data exclusion process was effective on both utilized datasets even though the learning task is more complex.
    \item The number of workers strongly affects the performance in terms of energy, accuracy, and loss if the data distribution is non-i.i.d. Simultaneously, for i.i.d, it is sufficient to select fewer workers to reach satisfactory accuracy.
    \item The model size has notable influences on energy expenditure, as seen from the simple and complex learning tasks conducted in this work. In the latter, the total energy consumption is much higher than the energy consumed for the learning task of the former. 
 \end{itemize}
\section{CONCLUSION}
\label{conclusion}
In this work, we propose a novel energy-efficient FEEL approach that contributes to significant improvements in terms of energy consumption. 
We take advantage of using a fine-grained data selection approach that excludes data samples that do not significantly contribute to the loss function.  
In our proposed approach, each worker tunes the received global model parameters while intelligently excluding the samples predicted with high probability based on a predefined threshold. Such samples do not introduce significant contributions to the learning model and their use can adversely impact energy consumption. The proposed approach tunes the transmission power and local CPU speed of workers in a FEEL system to enhance energy efficiency.  
We also exploit the FEEL round deadline constraint to optimize the uploading time and further reduce the expended energy.
Furthermore, we devise an iterative algorithm based on the Golden-section search method to obtain beamforming weights, allocated bandwidth, the local CPU frequency, and transmission power.   
Our experimental results demonstrate outstanding potential for reducing the total energy consumption of FEEL systems. 
Eventually, we show that energy can be significantly saved by adopting our proposed find-grained data selection approach. 
For future directions, accounting for the relationship between the threshold value and the intended learning task can be regarded. 
\bibliographystyle{IEEEtranTIE}
\bibliography{Ref}

\clearpage
\appendices
\section{PROOF OF THEOREM \ref{thm3}}
\label{appendix:theorm1}
For classification problems, softmax and cross-entropy are used thanks to their advantage of faster convergence, low computation,  and more accurate classification results.
\begin{figure}[!h]
\centering
  \includegraphics[width=0.8\linewidth]{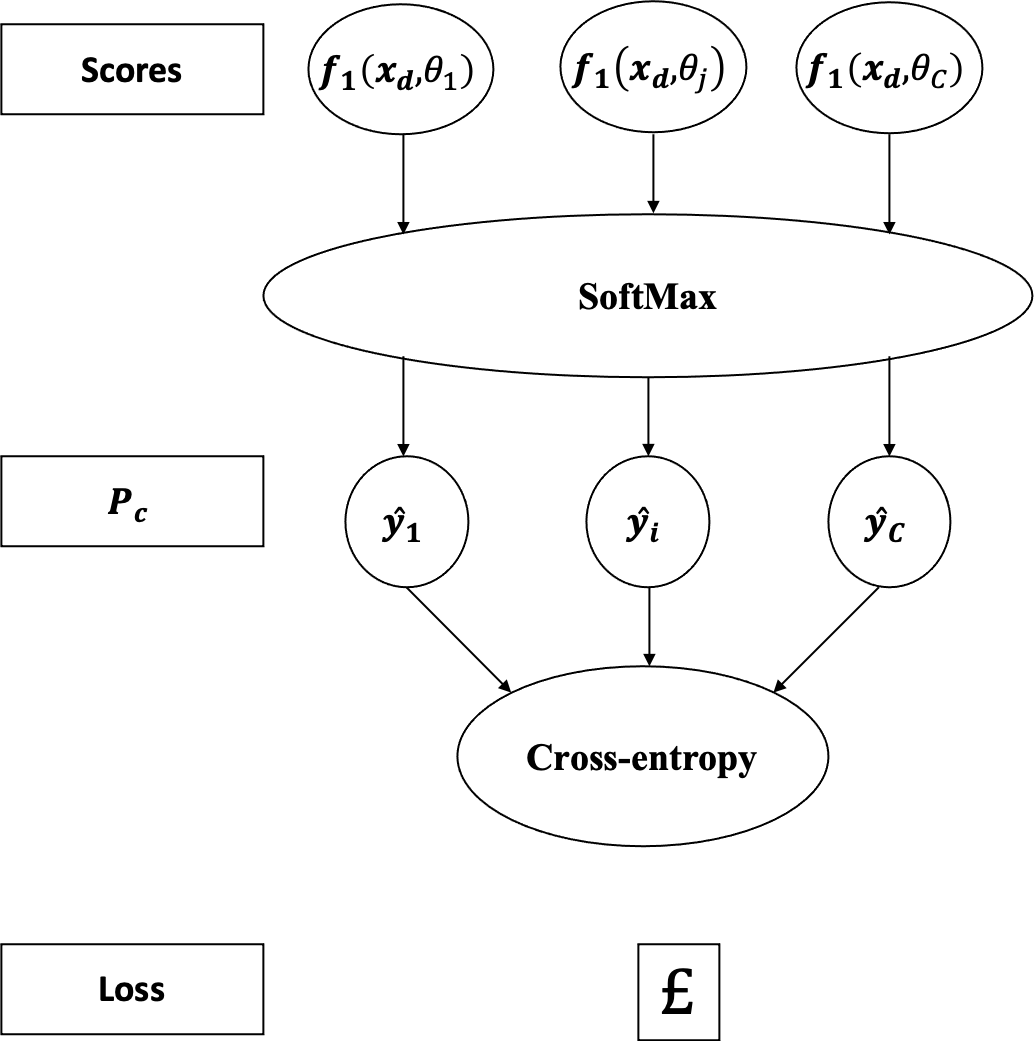}
\caption{Soft-max and cross-entropy (Relations and outputs).}
\label{fig:cross-soft}
\end{figure}
It is worth noting that, we consider only one sample to simplify the presentation as in \Cref{fig:cross-soft}. To start with, the softmax layer for every output $\hat{y_c}$ is defined as:
\begin{equation}
\label{eq:softmax}
    \hat{y} = \frac{\exp f_c(\x, \theta_c)}{\sum_{c=1}^{C} \exp f_c(\x, \theta_c)}
\end{equation}

Let $y_c = p^{(k)}({y}_{d}=c)$ and $log(\hat{y_i}) = \nabla_{\theta} \mathbb{E}[\log f_c(\x, \theta^{(k)}_{n-1})]$ we rewrite \eqref{eq:local_cross_entropy} as
\begin{equation}
\label{eq:simple_entropy}
    \mathcal{L}  = - \sum y_c log(\hat{y_i}) 
\end{equation}

Now, let's first drive the gradient of \eqref{eq:softmax} w.r.t score input (i.e., the gradient of a particular output w.r.t a particular score input). We have two cases of derivatives.
\begin{itemize}
\item Case 1: when input and output indices are the same ($c = j$) where $c ,j = \{1, \ldots, C\}$: 
\begin{equation}
\begin{aligned}
        \frac{\partial \hat{y_c}}{\partial \Psi } & = \nonumber\\
        & \frac{\exp \Psi  {\sum_{c=1}^{C} \exp \Psi } - \exp f_c(\x, \theta) \exp f_c(\x, \theta)}{({\sum_{c=1}^{C} \exp \Psi })^2} \nonumber \\
        & = \frac{\exp \Psi ( {\sum_{c=1}^{C} \exp \Psi } - \exp f_c(\x, \theta))}{({\sum_{c=1}^{C} \exp \Psi })({\sum_{c=1}^{C} \exp \Psi })} \nonumber \\
        &  = \frac{\exp \Psi }{({\sum_{c=1}^{C} \exp \Psi })} \nonumber \\ & . \frac{( {\sum_{c=1}^{C} \exp \Psi } - \exp f_c(\x, \theta))}{({\sum_{c=1}^{C} \exp \Psi })} \nonumber \\
        &  = \frac{\exp \Psi }{({\sum_{c=1}^{C} \exp \Psi })} \nonumber\\
        & .\Bigg( \frac{ {\sum_{c=1}^{C} \exp \Psi } }{{\sum_{c=1}^{C} \exp \Psi }} - \frac{\exp \Psi  }{{\sum_{c=1}^{C} \exp \Psi }}\Bigg)
        \label{eq:case1softmax}
 \end{aligned} 
\end{equation} 

where $\Psi = f_c(\x, \theta_c)$
    By substituting \eqref{eq:softmax} into first and last terms of \eqref{eq:case1softmax} we have:
    \begin{equation}
        \label{case1_updated}
         \frac{\partial \hat{y_c}}{\partial f_c(\x, \theta_c)} = \hat{y_c}.(1-\hat{y_c})
    \end{equation}
\item  Case 2: when input and output indices are not equal ($c \ne j$):
    \begin{align}
    \label{eq:case2softmax}
        \frac{\partial \hat{y_c}}{\partial f_c(\x, \theta_j)} & =  \frac{0 - \exp f_c(\x, \theta) \exp {f_c(\x, \theta_j)}}{({\sum_{c=1}^{C} \exp f_c(\x, \theta_c)})^2} \nonumber \\
        & = -\Bigg(\frac{\exp f_c(\x, \theta_c) }{({\sum_{c=1}^{C} \exp f_c(\x, \theta_c)})}.\nonumber \\ & \quad \frac{\exp {f_c(\x, \theta_j)} }{({\sum_{c=1}^{C} \exp f_c(\x, \theta_c)})}  \Bigg)
    \end{align}
Similarly, by substituting \eqref{eq:softmax} into first and second terms of \eqref{eq:case2softmax} we have:
    \begin{equation}
     \label{case2_updated}
         \frac{\partial \hat{y_c}}{\partial f_c(\x, \theta_j)} = -(\hat{y_c}.\hat{y_j})  
    \end{equation}
\end{itemize}

Next, we link \eqref{case1_updated} and \eqref{case2_updated} to the derivative of cross-entropy w.r.t a particular output.
\begin{align}
    \label{eq:cross-entropy_y}
        \frac{\partial \mathcal{L}}{\partial \hat{y_c}} = y_c \frac{1}{\hat{y_c}}
    \end{align}
Further, we derive the cross-entropy w.r.t input of softmax ,$f_c(\x, \theta_c)$, as we want to combine the derivative for both the input and output of the softmax layer (i.e., back-propagation and chain rule).
    \begin{align}
    \label{eq:cross-entropy_x}
        \frac{\partial \mathcal{L}}{\partial {f_c(\x, \theta_c)}} = - \sum_{c \ne j} y_i \frac{1}{\hat{y_c}} \frac{\partial \hat{y_c}}{\partial f_c(\x, \theta_j)} - y_j \frac{1}{\hat{y_j}} \frac{\partial \hat{y_j}}{\partial f_c(\x, \theta_j)} 
    \end{align}
 By substituting \eqref{case1_updated} and \eqref{case2_updated} into right hand side of \eqref{eq:cross-entropy_y}, we have:
 \begin{align}
    \label{eq:cross-entropy_x2}
        \frac{\partial \mathcal{L}}{\partial {f_c(\x, \theta_c)}} & =  - \sum_{c \ne j} y_c \frac{1}{\hat{y_c}} (-\hat{y_c}\hat{y_j}) - y_j \frac{1}{\hat{y_j}} \hat{y_j}.(1-\hat{y_j})\nonumber \\
        & =  \sum_{c \ne j} y_c \hat{y_j} - y_j + y_j \hat{y_j}\nonumber \\
      &  =   \sum_{c \ne j} y_c \hat{y_j}  + y_j \hat{y_j} - y_j \nonumber\\
      & =  \sum_{\forall c} y_c \hat{y_j} - y_j
    \end{align}

Moreover, we have the fact that the class labels are one-hot encoded; thus, $\sum_{\forall c} y_c = 1$ (both the input and output have the same indices), hence, we rewrite \eqref{eq:cross-entropy_x2} as follow:
\begin{align}
     \frac{\partial \mathcal{L}}{\partial \hat{f_c(\x, \theta_c)}} = & \hat{y_c} - y_c\nonumber \\
     & = \nabla_{\theta} \mathbb{E}[\log f_c(\x, \theta^{(k)}_{n-1})] - p^{(k)}({y}_{d}=c) 
\label{eq:proof_them2}     
\end{align}
From \eqref{eq:proof_them2}, we can note that if the prediction probability of a given input sample, the first term in~\eqref{eq:proof_them2}, becomes closer to $1$, its contribution to the loss function becomes less significant.  \qed
\end{document}